\documentclass[twoside]{article} 
\usepackage[accepted]{aistats2017}

\usepackage{multibib}
\usepackage{times}
\usepackage{subfigure}
\usepackage{amsmath}
\usepackage{amsthm}
\usepackage{amsfonts}
\usepackage{amssymb}
\usepackage{graphicx}
\usepackage{hyperref}
\usepackage{natbib}
\usepackage{algorithmic}
\usepackage{algorithm}

\usepackage{color}
\usepackage{xcolor}
\usepackage{hyperref}
\usepackage{graphicx}
\usepackage{multirow}

\newcommand{\tvx}{\widetilde{\vx}}

\newcommand{\pa}[0]{\textrm{pa}}

\newcommand{\dkls}[3]{\mathbb{D}_{KL}^{#1}[#2 \, \|\, #3]}

\newcommand{\deriv}[2]{\frac{\partial{#1}}{\partial{#2}}}
\newcommand{\secondderiv}[2]{\frac{\partial^2{#1}}{\partial{#2}^2}}
\newcommand\cut[1]{}

\newcommand{\elbofinal}{\mathcal{L}}

\newcommand{\tvlambda}{\widetilde{\vlambda}}
\newcommand{\tlambda}{\widetilde{\lambda}}

\newcommand{\tp}{\tilde{p}}

\newcommand{\squishlist}{
   \begin{list}{$\bullet$}
    { \setlength{\itemsep}{0pt}      \setlength{\parsep}{3pt}
      \setlength{\topsep}{3pt}       \setlength{\partopsep}{0pt}
      \setlength{\leftmargin}{1.5em} \setlength{\labelwidth}{1em}
      \setlength{\labelsep}{0.5em} } }

\newcommand{\squishlisttwo}{
   \begin{list}{$\bullet$}
    { \setlength{\itemsep}{0pt}    \setlength{\parsep}{0pt}
      \setlength{\topsep}{0pt}     \setlength{\partopsep}{0pt}
      \setlength{\leftmargin}{2em} \setlength{\labelwidth}{1.5em}
      \setlength{\labelsep}{0.5em} } }

\newcommand{\squishend}{
    \end{list}  }
 
{}
\newtheorem{thm}{Theorem}{}
{}
{}
\newtheorem{lemma}{Lemma}{}
{}

\newcommand{\half}{\mbox{$\frac{1}{2}$}}

\newcommand{\real}{\mbox{$\mathbb{R}$}}

\newcommand{\rnd}[1]{\left(#1\right)}
\newcommand{\sqr}[1]{\left[#1\right]}
\newcommand{\crl}[1]{\left\{#1\right\}}
\newcommand{\ang}[1]{\langle#1\rangle}
\newcommand{\myexpect}{\mathbb{E}}

\newcommand{\Ga}{\mbox{Ga}}
\newcommand{\gauss}{\mbox{${\cal N}$}}

\newcommand{\myvec}[1]{\mbox{$\mathbf{#1}$}}
\newcommand{\myvecsym}[1]{\mbox{$\boldsymbol{#1}$}}

\newcommand{\veta}{\mbox{$\myvecsym{\eta}$}}

\newcommand{\vmu}{\mbox{$\myvecsym{\mu}$}}
\newcommand{\vlambda}{\mbox{$\myvecsym{\lambda}$}}

\newcommand{\vphi}{\mbox{$\myvecsym{\phi}$}}

\newcommand{\vtheta}{\mbox{$\myvecsym{\theta}$}}

\newcommand{\va}{\mbox{$\myvec{a}$}}
\newcommand{\vb}{\mbox{$\myvec{b}$}}

\newcommand{\vg}{\mbox{$\myvec{g}$}}

\newcommand{\vm}{\mbox{$\myvec{m}$}}

\newcommand{\vw}{\mbox{$\myvec{w}$}}

\newcommand{\vx}{\mbox{$\myvec{x}$}}

\newcommand{\vy}{\mbox{$\myvec{y}$}}

\newcommand{\vz}{\mbox{$\myvec{z}$}}

\newcommand{\vC}{\mbox{$\myvec{C}$}}

\newcommand{\vI}{\mbox{$\myvec{I}$}}

\newcommand{\vK}{\mbox{$\myvec{K}$}}

\newcommand{\vV}{\mbox{$\myvec{V}$}}
\newcommand{\vW}{\mbox{$\myvec{W}$}}
\newcommand{\vX}{\mbox{$\myvec{X}$}}

\newcommand{\vY}{\mbox{$\myvec{Y}$}}
\newcommand{\vZ}{\mbox{$\myvec{Z}$}}

\newcommand{\union}{\mbox{$\cup$}}

\newcommand{\trace}{\mbox{Tr}}

\newcommand{\be}{\begin{equation}}
\newcommand{\ee}{\end{equation}}
\newcommand{\bea}{\begin{eqnarray}}
\newcommand{\eea}{\end{eqnarray}}
\newcommand{\beaa}{\begin{eqnarray*}}
\newcommand{\eeaa}{\end{eqnarray*}}

\usepackage{mdframed}
\newmdenv[
  topline=false,
  bottomline=false,
  rightline=false,
  skipabove=\topsep,
  ]{siderules}

\begin{document}

\runningtitle{Conjugate-Computation Variational Inference}

%
\runningauthor{Khan, Lin}

\twocolumn[

\aistatstitle{Conjugate-Computation Variational Inference :
Converting Variational Inference in Non-Conjugate Models to Inferences in Conjugate Models}

\aistatsauthor{ Mohammad Emtiyaz Khan \And Wu Lin}

\aistatsaddress{Center for Advanced Intelligence Project (AIP)\\ RIKEN, Tokyo \And Center for Advanced Intelligence Project (AIP) \\ RIKEN, Tokyo} ]

\begin{abstract}
   Variational inference is computationally challenging in models that contain both conjugate and non-conjugate terms.
Methods specifically designed for conjugate models, even though computationally efficient, find it difficult to deal with non-conjugate terms.
On the other hand, stochastic-gradient methods can handle the non-conjugate terms but they usually ignore the conjugate structure of the model which might result in slow convergence.
In this paper, we propose a new algorithm called Conjugate-computation Variational Inference (CVI) which brings the best of the two worlds together -- 
it uses conjugate computations for the conjugate terms and employs stochastic gradients for the rest.
We derive this algorithm by using a stochastic mirror-descent method in the mean-parameter space, and then expressing each gradient step as a variational inference in a conjugate model.
We demonstrate our algorithm's applicability to a large class of models and establish its convergence. 
Our experimental results show that our method converges much faster than the methods that ignore the conjugate structure of the model.

\end{abstract}

\section{Introduction}
In this paper, we focus on designing efficient variational inference algorithms for models that contain both conjugate and non-conjugate terms, e.g., models such as Gaussian process classification \citep{Kuss05}, correlated topic models \citep{blei2007correlated}, exponential-family Probabilistic PCA \citep{mohamed2009bayesian}, large-scale multi-class classification \citep{genkin2007large}, Kalman filters with non-Gaussian likelihoods \citep{rue2005gaussian}, and deep exponential-family models \citep{ranganath2015deep}.
Such models are widely used in machine learning and statistics, yet variational inference on them remains computationally challenging. 

The difficulty lies in the non-conjugate part of the model.
In the traditional Bayesian setting, when the prior distribution is conjugate to the likelihood, the posterior distribution is available in closed-form and can be obtained through simple computations.
For example, in a conjugate-exponential family, computation of the posterior distribution can be achieved by simply adding the sufficient statistics of the likelihood to the natural parameter of the prior.
In this paper, we refer to such computations as \emph{conjugate computations} (an example is included in the next section).

These types of conjugate computations have been used extensively in variational inference, primarily due to their computational efficiency.
For example, the variational message-passing (VMP) algorithm proposed by \cite{winn2005variational} uses conjugate computations within a message-passing framework.
Similarly, stochastic variational inference (SVI) builds upon VMP and enables large-scale inference by employing stochastic methods \citep{hoffman2013stochastic}.

Unfortunately, the computational efficiency of these methods is lost when the model contains non-conjugate terms.
For example, the messages in VMP lose their convenient exponential-family form and become more complex as the algorithm progresses.
Additional approximations for the non-conjugate terms can be used, e.g. those discussed by \cite{winn2005variational} and \cite{WangBlei}, but such approximations usually result in a performance loss \citep{honkela2004unsupervised,emtThesis}.
Other existing alternatives, such as the non-conjugate VMP method of \cite{knowles2011non} and the expectation-propagation method of \cite{Minka01b}, also require carefully designed quadrature methods to approximate the non-conjugate terms, and suffer from convergence problems and numerical issues.

Recently, many stochastic-gradient (SG) methods have been proposed to deal with this issue \citep{ranganath2013black, salimans2013fixed, titsias2014doubly}. 
An advantage of these approaches is that they can be used as a black-box and applied to a wide-variety of inference problems.
However, these methods usually do not directly exploit conjugacy, e.g., during the stochastic-gradient computation.
This might lead to several issues, e.g., their updates may depend on the parameterization of the variational distribution, the number of variational parameters might be too large, and the updates might converge slowly.

In this paper, we propose an algorithm which brings the best of the two worlds together -- it uses stochastic gradients for the non-conjugate terms, while retaining the computational efficiency of conjugate computations on the conjugate terms.
We call our approach Conjugate-computation Variational Inference (CVI).
Our main proposal is to use a stochastic \emph{mirror-descent  method} in the \emph{mean-parameter space} which differs from many existing methods that use a stochastic \emph{gradient-descent method} in the \emph{natural-parameter space}.
Our method has a natural advantage over these methods -- gradient steps in our method can be implemented by using conjugate computations.

We demonstrate our approach on two classes of non-conjugate models. The first class contains models which can be split into a conjugate part and a non-conjugate part. For such models our gradient steps can be expressed as a Bayesian inference in a conjugate model. The second class of models additionally allows conditionally-conjugate terms. For this model class, our gradient steps can be written as a message passing algorithm where VMP or SVI is used for the
conjugate part while stochastic gradients are employed for the rest.
Our algorithm conveniently reduces to VMP when the model is conjugate.

We also prove convergence of our algorithm and establish its connections to many existing approaches. We apply our algorithm to many existing models and demonstrate that our updates can be implemented using variational inference in conjugate models. Empirical results on many models and datasets show that our method converges much faster than the methods that ignore the conjugate structure of the model. The code to reproduce results of this paper is available at {\footnotesize \bf \url{https://github.com/emtiyaz/cvi/}}.


%

\section{Conjugate Computations} \label{sec:conj}
Given a probabilistic graphical model $p(\vy,\vz)$ with $\vy$ as the vector of observed variables and $\vz$ as the vector of latent variables, our goal in variational inference is to estimate a posterior distribution $p(\vz|\vy)$.
When the prior distribution $p(\vz)$ is \emph{conjugate}\footnote{A prior distribution is said to be conjugate when it takes the same functional form as the likelihood. An exact definition is given in Appendix \ref{conj:def}.} to the likelihood $p(\vy|\vz)$, the posterior distribution is available in closed form and can be obtained through simple computations which we refer to as the \emph{conjugate computations}.
For example, consider the following exponential-family prior distribution:$ p(\vz) =  h(\vz)\exp\sqr{\ang{\vphi(\vz), \veta} - A(\veta)}$,
where $\veta$ is the natural parameter, $\vphi$ is the sufficient statistics,
$\ang{\cdot,\cdot}$ is an inner product, $h(\vz)$ is the base measure, and 
$A(\veta)$ is the log-partition function. 
When the likelihood is conjugate to the prior, we can express the likelihood in the same form as the prior with respect to $\vz$, as shown below:
\begin{align}
   p(\vy|\vz) &= \exp\sqr{\ang{\vphi(\vz), \veta_{yz}(\vy)} -f_y(\vy)} , \label{eq:conjexp1}
\end{align}
where $\veta_{yz}$ and $f_y$ are functions that depend on $\vy$ only.
In such cases, the posterior distribution takes the same exponential form as $p(\vz)$ and its natural parameter can be obtained by simply adding the natural parameters $\veta$ of the prior to the function $\veta_{yz}(\vy)$ of the likelihood:
\begin{align}
  p(\vz|\vy) &\propto h(\vz) \exp\sqr{\ang{\vphi(\vz), \veta + \veta_{yz}(\vy)}} . \label{eq:conjexp3}
\end{align}
This is a type of conjugate computation. Such conjugate computations are extensively used in Bayesian inference for conjugate models, as well as in variational inference for conditionally-conjugate models in algorithms such as variational message passing \citep{winn2005variational} and expectation propagation \citep{Minka01b}.

\section{Non-Conjugate Variational Inference} \label{sec:ncvi}
When the model also contains non-conjugate terms, variational inference becomes computationally challenging. In variational inference, we obtain a fixed-form variational approximation $q(\vz|\vlambda)$, where $\vlambda$ is the variational parameter, by maximizing a lower bound to the marginal likelihood:
\begin{align}
   \max_{\lambda\in\Omega} \elbofinal(\vlambda) := \myexpect_q [\log p(\vy,\vz) - \log q(\vz|\vlambda)] ,
  \label{eq:lb}
\end{align}
where $\Omega$ is the set of valid variational parameters.
Non-conjugate terms might make the lower-bound optimization challenging, e.g., by making it intractable. For example, Gaussian Process (GP) models usually employ a non-Gaussian likelihood and the variational lower bound becomes intractable, as discussed below.
\begin{siderules}
   {\bf GP Example: } Consider a GP model for $N$ input-output pairs $\{y_n,\vx_n\}$ indexed by $n$.
   Let $z_n := f(\vx_n)$ be the latent function drawn from a GP with mean 0 and covariance $\vK$. Given $z_n$, we use a non-Gaussian likelihood $p(y_n|z_n)$ to model the output $y_n$.
   The joint distribution is shown below:
   \begin{align}
   p(\vy,\vz) = \sqr{\prod_{n=1}^N p(y_n|z_n)} \gauss(\vz|0,\vK) .
   \label{eq:gp_joint}
   \end{align}
   It is a common practice to approximate the posterior distribution by a Gaussian distribution $q(\vz|\vlambda) := \gauss(\vz|\vm,\vV)$ whose  mean $\vm$ and covariance $\vV$ we need to estimate \citep{Kuss05}.
   By substituting the joint-distribution \eqref{eq:gp_joint} in the lower bound \eqref{eq:lb}, we get the following lower bound:
   \begin{align}
      \sum_n \myexpect_q\sqr {\log p(y_n|z_n) } - \dkls{}{q(\vz|\vlambda)}{\gauss(\vz|0,\vK)} ,
      \label{eq:lb_gp}
   \end{align}
   where $\mathbb{D}_{KL}$ is the Kullback-Leibler divergence.
   This lower bound is intractable for most \emph{non-Gaussian} likelihoods because the expectation $\myexpect_q\sqr{\log p(y_n|z_n)}$ usually does not have a closed-form expression, e.g., when $p(y_n|z_n)$ is a logistic or probit likelihood. 
\end{siderules}

Despite its intractability, the lower bound can still be optimized by using a stochastic-gradient method, e.g., the following stochastic-gradient descent (SGD) algorithm:
\begin{align}
   \vlambda_{t+1} = \vlambda_t + \rho_t \widehat{\nabla}_{\lambda} \elbofinal (\vlambda_t) ,
   \label{eq:sgd}
\end{align}
where $t$ is the iteration number, $\rho_t$ is a step size, and $\widehat{\nabla}_{\lambda} \elbofinal (\vlambda_t) := \widehat{ \partial \elbofinal} / \partial \vlambda$ is a stochastic gradient of the lower bound at $\vlambda = \vlambda_t$.
The advantage of this approach is that it can be used as a black-box method and applied to a wide-variety of inference problems \citep{ranganath2013black}. 

Despite its generality and scalability, there are major issues with the SG method.
The conjugate terms in the lower bound might have a closed-form expression and may not require any stochastic approximations.
A naive application of the SG method might ignore this.
Another issue is that the efficiency and rate of convergence might depend on the parameterization used for the variational distribution $q(\vz|\vlambda)$. Some parameterizations might have simpler updates than others but it is usually not clear how to find the best one for a given model. We discuss these issues below for the GP example.
\begin{siderules}
   {\bf GP Example (issues with SGD):} 
   In GP models, it might seem that the number of variational parameters should be in $O(N^2)$, e.g., if we use $\{\vm,\vV\}$.
   However, as \cite{Opper:09} show, there are only $O(N)$ free parameters.
   Therefore, choosing a naive parameterization might be an order of magnitude slower than the best option (see Appendix \ref{sec:issueswithsgd} for more details on the inefficiency of SGD).
   In fact, as shown in \cite{Khan12nips}, choosing a good parameterization is a difficult problem in this case and a naive selection might make the problem more difficult than it needs to be.
\end{siderules}
Our algorithm, derived in the next sections, does not suffer from such issues, e.g., for the GP example our algorithm will conveniently express each gradient step as predictions in a GP regression model which naturally has an $O(N)$ number of free parameters. 
In the results section, we will see that this results in a fast convergent algorithm.

\section{Conjugate-Computation Variational Inference (CVI)}
\label{sec:cvi1}
We now derive the CVI algorithm that uses stochastic gradients for the non-conjugate terms, while retaining the computational efficiency of conjugate computations for the conjugate terms.
We will use a stochastic \emph{mirror-descent  method} in the \emph{mean-parameter space} and show that its gradient steps can be implemented by using conjugate computations. 
This will fix the issues of stochastic-gradient methods but maintain the computational efficiency of conjugate computations.

Our approach relies on the following two assumptions:

{\bf Assumption 1 [minimality] :} \emph{The variational distribution $q(\vz|\vlambda)$ is a ``minimal" exponential-family distribution:
\begin{align}
   q(\vz|\vlambda) =  h(\vz)\exp\crl{\ang{\vphi(\vz), \vlambda} - A(\vlambda)} ,
   \label{eq:approx_post}
\end{align}
with $\vlambda$ as its natural parameters.
}

The minimal\footnote{A summary of exponential family is given in Appendix \ref{app:basics}.} representation implies that there is a one to one mapping between the mean parameter $\vmu := \myexpect_q [\vphi(\vz)]$ and the natural parameter $\vlambda$.
Therefore, we can express the lower-bound optimization as a maximization problem over $\vmu \in \mathcal{M}$, where $\mathcal{M}$ is the set of valid mean parameters.
We denote the new objective function by $\widetilde{\mathcal{L}}(\vmu) := \elbofinal(\vlambda)$.

{\bf Assumption 2 [conjugacy] :} \emph{We assume that the joint distribution contains some terms, collectively denoted by $\tp_c$, which take the same form as $q$ with respect to $\vz$, i.e.,
\begin{align}
   \tp_c(\vy,\vz) \propto h(\vz)\exp\crl{\ang{\vphi(\vz), \veta}},
\end{align}
where $\veta$ is a known parameter vector.
We call $\tp_c$ as the conjugate part of the model.
We denote the non-conjugate terms by $\tp_{nc}$ giving us the following partitioning of the joint distribution: $p(\vy,\vz) \propto \tp_{nc}(\vy,\vz) \tp_c(\vy,\vz)$. These terms can be unnormalized with respect to $\vz$.
}

We can always satisfy this assumption, e.g., by trivially setting $\veta=0$ and $\tp_{nc} = p(\vy,\vz)/h(\vz)$. However, since there is no conjugacy in this formulation, our algorithm might not be able to gain additional computational efficiency over the SG methods.


We now derive the CVI algorithm. We build upon an equivalent formulation of \eqref{eq:sgd} which expresses the gradient step as the maximization of a local approximation: 
\begin{align}
   \vlambda_{t+1} = \arg\max_{\lambda \in \Omega} \left\langle \vlambda, \widehat{\nabla}_{\lambda} \elbofinal (\vlambda_t) \right\rangle 
   - \frac{1}{2 \rho_t} \| \vlambda - \vlambda_t \|_2^2 ,
   \label{eq:sgd2}
\end{align}
where $\|\cdot\|_2$ is the Euclidean norm and $\Omega$ is the set of valid natural parameters. 
By taking the derivative and setting it to zero, we recover the SGD update of \eqref{eq:sgd} which establishes the equivalence.

Instead of using the above SGD update in the natural-parameter space, we propose to use a stochastic mirror-descent update in the mean-parameter space. The mirror-descent algorithm \citep{nemirovskii1983problem} replaces the Euclidean geometry in \eqref{eq:sgd2} by a \emph{proximity} function such as a Bregman divergence \citep{raskutti2015information}. 
%
We propose the following mirror-descent algorithm:
\begin{align}
   \vmu_{t+1} = \arg\max_{\mu \in \mathcal{M}} \left\langle \vmu, \widehat{\nabla}_{\mu} \widetilde{\elbofinal} (\vmu_t) \right\rangle 
   - \frac{1}{\beta_t} \mathbb{B}_{A^*}(\vmu \| \vmu_t) ,
   \label{eq:prox}
\end{align}
where $A^*(\vmu)$ is the convex-conjugate\footnote{Definitions of convex-conjugate and Bregman divergence is given in Appendix \ref{app:basics}.} of the log-partition function $A(\vlambda)$, $\mathbb{B}_{A^*}$ is the Bregman divergence defined by $A^*$ over $\mathcal{M}$, and $\beta_t>0$ is the step-size.
The following theorem establishes that \eqref{eq:prox} can be implemented by using a Bayesian inference in a conjugate model.

\begin{thm}
   \label{thm:main}
   Under Assumption 1 and 2, the update \eqref{eq:prox} is equivalent to the Bayesian inference in the following conjugate model:
   \begin{align}
      q(\vz|\vlambda_{t+1}) \,\, \propto \,\,  e^{\ang{\boldsymbol{\phi}(\mathbf{z}), \widetilde{\boldsymbol{\lambda}}_t }} \,  \tp_c(\vy,\vz) , 
       \label{eq:update_exp1}
   \end{align}
   whose natural parameter can obtained by conjugate computation: $\vlambda_{t+1} = \tvlambda_t + \veta$ where $\tvlambda_t$ is the natural parameter of the exponential-family approximation to $\tp_{nc}$ and can be obtained recursively as follows:
   \begin{align}
      \tvlambda_t = (1-\beta_t) \tvlambda_{t-1} + \beta_t \, \,   \widehat{\nabla}_{\mu} \myexpect_{q} \sqr{ \log \tp_{nc}}  \rvert_{\mu=\mu_t}, \label{eq:recursion}
   \end{align}
   with $\vlambda_0 = 0$ and $\vlambda_1 = \veta$.
\end{thm}
A proof is given in Appendix \ref{app:proof_thm1}. The update \eqref{eq:update_exp1} replaces the non-conjugate term by an exponential-family approximation whose natural parameter $\tvlambda_t$ is a weighted sum of the gradients of the non-conjugate term $\tp_{nc}$. 
The resulting algorithm, which we refer to as the conjugate-computation variational inference (CVI) algorithm, is summarized in Algorithm \ref{alg:cvi1}. As desired, our algorithm computes stochastic-gradients only for the non-conjugate terms, as shown in Step \ref{st:project}. Given this gradient, in Step \ref{st:infer}, the new variational parameter is obtained by using a conjugate computation by simply adding the natural parameters.

Note that, even though we proposed an update in the mean-parameter space, conjugate computations in Step \ref{st:infer} are performed in the natural-parameter space. The mean parameter is required only during the computation of stochastic gradients in Step \ref{st:project}. For the GP example, these updates are conveniently expressed as predictions in a GP regression model which naturally has an $O(N)$ number of free parameters, as discussed next.

\begin{algorithm}[t]
\caption{CVI for exponential-family approximations.}
\label{alg:cvi1} 
\begin{algorithmic}[1]
    \STATE Initialize $\tvlambda_0 = 0$ and $\vlambda_1 = \veta$.
    \FOR{$t=1,2,3,\ldots$,}
      \STATE $\tvlambda_t = (1-\beta_t) \tvlambda_{t-1} + \beta_t \, \,   \widehat{\nabla}_{\mu} \myexpect_{q} \sqr{ \log \tp_{nc}}  \rvert_{\mu=\mu_t}$.
       \label{st:project}
       \STATE $\vlambda_{t+1} = \tvlambda_t + \veta$.

       \label{st:infer}
    \ENDFOR
\end{algorithmic}
\end{algorithm}

\begin{siderules}
   {\bf GP Example (CVI updates):} For the GP model, the non-conjugate part $\tp_{nc}$ is $\prod_n p(y_n|z_n)$. Both Assumption 1 and 2 are satisfied since $q$ is a Gaussian and the GP prior is conjugate to it.

   For Step \ref{st:project}, we need to compute the gradient with respect to $\vmu$. Since $\tp_{nc}$ factorizes, we can compute the gradient of each term $\myexpect_q [\log p(y_n|z_n)]$ separately. This term depends on the marginal $q(z_n)$ which has two mean parameters:
$\mu_n^{(1)} := m_n$ and $\mu_n^{(2)} := V_{nn} + m_n^2$,
where $V_{nn}$ is the $n$'th diagonal element of $\vV$. The gradients can be computed using the Monte Carlo (MC) approximation as shown by \cite{Opper:09} (details are given in Appendix \ref{sec:gp_grad}). Let's denote these gradients at iteration $t$ by $\hat{g}_{n,t}^{(1)}$ and $\hat{g}_{n,t}^{(2)}$. Using these, Step \ref{st:project} of Algorithm \ref{alg:cvi1} can be written as follows:
\begin{align}
   \tlambda_{n,t}^{(i)} = (1-\beta_t) \tlambda_{n,t-1}^{(i)} + \beta_t \hat{g}_{n,t}^{(i)} ,
\end{align}
for $i=1,2$ and $n = 1,2,\ldots,N$. These are the natural parameters for a Gaussian approximation of $p(y_n|z_n)$.
Using them in \eqref{eq:update_exp1} we obtain the following conjugate model:
\begin{align}
   q(\vz|\vlambda_{t+1}) \propto \sqr{\prod_{n=1}^N e^{z_n \tlambda_{n,t}^{(1)} + z_n^2 \tlambda_{n,t}^{(2)}} } \gauss(\vz|0,\vK) . \nonumber
\end{align}
This update can be done by using a conjugate computation which in this case corresponds to predictions in a GP regression model (since we only need to compute the mean parameter $\mu_n^{(1)}$ and $\mu_n^{(2)}$ for all $n$). We see that the only free parameters to be computed are $\tlambda_{n,t}^{(1)}$ and $\tlambda_{n,t}^{(2)}$, therefore the number of parameters is in $O(N)$. We naturally end up with the optimal number of parameters and avoid the computation of the full covariance matrix $\vV$. Both of these give a huge computational saving over a naive SGD method.
\end{siderules}
The previous example shows that variational inference in non-conjugate GP models can be done by solving a sequence of GP regression problems. In Appendix \ref{app:ex_cvi}, we give many more such examples. In Appendix \ref{app:glm}, we show that the variational inference in a generalized linear model (GLM) can be implemented by using Bayesian inference in linear regression model. Similarly, in Appendix \ref{app:kalman}, we show a Kalman filter model with GLM likelihood can be implemented by using Bayesian inference in the standard Kalman filter.
We also give examples involving Gamma variational distribution in Appendix \ref{app:gamma1}.

It is also possible to use a ``doubly" stochastic approximation \citep{titsias2014doubly} where, in addition to the MC approximation, we also randomly pick factors from both $\tp_{nc}$ and $\tp_c$. As discussed below, this will result in a huge reduction in computation in Step \ref{st:project} of Algorithm \ref{alg:cvi1}, e.g., bringing the number of stochastic-gradient computations to $O(1)$ from $O(N)$ in the GP example.

\begin{siderules}
   {\bf GP Example (doubly-stochastic CVI):} We can use a doubly-stochastic scheme over $\tp_{nc}$ since it factorizes over $n$. We sample one term (or pick a mini-batch) and compute its stochastic gradient. In our algorithm, this translates to modifying only the selected example's $\tlambda_{n,t}^{(i)}$. Denoting the selected sample index by $n_t$, this can be expressed as follows:
\begin{align}
   \tlambda_{n,t}^{(i)} = (1-\beta_t)  \tlambda_{n,t-1}^{(i)} + \delta_{n=n_t} \,\,  \beta_t  N\hat{g}_{n,t}^{(i)} \, ,
\end{align}
where $\delta_{n=n_t}$ is an indicator function which is 1 only when $n=n_t$.
The number of stochastic gradient computation is therefore in $O(1)$ instead of $O(N)$.
Computation is further reduced since at each iteration $t$ we only need to compute \emph{one} mean parameter corresponding to the marginal of $z_{n_{t}}$. This is much more efficient than a SGD update where we have to explicitly store $\vV_t$.
We can also reuse computations between iterations since updates in iteration $t$ differs from iteration $t-1$ only at one example $n=n_t$.
%
\end{siderules}

Another attractive property of our algorithm is that it can handle constraints on the variational parameters without much computational overhead, e.g., the covariance matrix in GP example will be positive definite as long as the Gaussian approximation of all non-conjugate term is valid. We can always make sure that this is the case by using rejection sampling inside the stochastic approximation.

One requirement for our algorithm is that we should be able to compute stochastic gradients with respect to the mean parameter.
For distributions such as Gaussian and Multinomial, this can be done analytically, e.g. see Appendix \ref{sec:gp_grad} for Gaussian distribution.
For other distributions, such as Gamma, this quantity might be difficult to compute directly.
We propose to build stochastic approximations by using the following identity:
\begin{align}
   \deriv{f}{\vmu} = \sqr{ \deriv{\vmu}{\vlambda} }^{-1} \deriv{f}{\vlambda} =  \vC_{\lambda}^{-1} \deriv{f}{\vlambda} ,
   \label{eq:grad_mu}
\end{align}
where $\vC_{\lambda}$ is the Fisher information matrix and $f$ is the function whose gradient we want to compute.
We compute stochastic approximations of $\vC_\lambda$ and $\partial f/\partial \vlambda$ separately, and then solve the equation to get the gradients. More details are given in Appendix \ref{app:gradmu}.
Our proposal is very similar to the one discussed by \cite{salimans2013fixed} where an approximation to $\vC_\lambda$ is obtained by averaging over iterations. 
The advantage of our proposal is that we do not have to explicitly store or form the Fisher information matrix, rather only solve a linear system.

The convergence of our algorithm is guaranteed under mild conditions discussed in \cite{khan2016faster}. The update \eqref{eq:prox} converges to a local optimum of the lower bound $\mathcal{L}(\vlambda)$ under the following conditions: $\widetilde{\elbofinal}(\vmu)$ is differentiable and its gradient is $L-$Lipschitz-continuous,
the stochastic approximation is unbiased  and has bounded variance, and the function $A^*(\vmu)$ is continuously differentiable and strongly convex with respect to the $L_2$ norm.
An exact statement of convergence is given in Proposition 3 of  \cite{khan2016faster}.

\section{CVI for Mean-Field Approximation}
\label{sec:cvi2}
We now extend our algorithm to models that also allow \emph{conditional-conjugacy}.
This class of models is bigger than the one considered in the previous section, however we will restrict the posterior approximation to a mean-field approximation which is a stricter assumption than the one used in the previous section.
The algorithm presented in this section is a generalization of VMP and SVI to non-conjugate models. We will see that the new algorithm differs only slightly from these previous algorithms and reduces to them when the model does not contain any non-conjugate terms.

Consider a Bayesian network over $\vx = \{\vy,\vz\}$ where $\vy$ is the vector of observed nodes $\vy_n$ for $n=1,2,\ldots,N$, and $\vz$ is the vector of latent nodes $\vz_i$ for $i=1,2,\ldots, M$. The joint distribution over $\vx$ is given as follows:
\begin{align}
    p(\vx) = \prod_{a=1}^{M+N} p(\vx_a|\vx_{\pa_a}) ,
    \label{eq:bayesnet}
\end{align}
where $\vx_{\pa_a}$ is the set of parent nodes for variable $\vx_a$.
We will refer to a term $p(\vx_a|\vx_{\pa_a})$ in $p(\vx)$ as the factor $a$.

Similar to the previous section, we make the following two assumptions for the CVI algorithm derived in this section.

{\bf Assumption 3 [mean-field + minimality] :} \emph{We assume that $q(\vz) = \prod_i q_i(\vz_i)$ with each factor being a minimal exponential-family distribution:
\begin{align}
   q_i(\vz_i|\vlambda_i) := h_i(\vz_i) \exp\sqr{ \ang{\vphi_i(\vz_i), \vlambda_i} - A_i(\vlambda_i)} .
   \label{eq:qi}
\end{align}
}
We denote the vector of mean parameters $\vmu_i$ by $\vmu$ and the vector of natural parameters $\vlambda_i$ by $\vlambda$. Due to minimality, we can rewrite the lower bound in terms of $\vmu$, for which we use the same notation $\widetilde{L}(\vmu)$ as in the previous section.

For the next assumption, we define $\mathbb{N}_i$ to be the set containing the node $\vz_i$ and all its children. We define $\vx_{/i}$ to be the set of all nodes $\vx$ except $\vz_i$. Similarly, given a factor $p(\vx_a|\vx_{pa_a})$ for node $\vx_a\in \mathbb{N}_i$, we define $\vx_{a/i}$ to be the set of all nodes in the set $\vx_a \union \vx_{pa_a}$ except $\vz_i$. 

{\bf Assumption 4 [conditional-conjugacy] :} \emph{For each node $\vz_i$, we can split the following conditional distribution into a conjugate and a non-conjugate term as shown below:
\begin{align}
   &p(\vz_i|\vx_{/i}) \propto \prod_{a\in \mathbb{N}_i} p(\vx_a|\vx_{pa_a}) \label{eq:ass4}\\
   &\propto\,\, h_i(\vz_i) \prod_{a\in\mathbb{N}_i} \tp_{nc}^{a,i}(\vz_i,\vx_{a/i}) \, e^{\crl{\ang{ \boldsymbol{\phi}_i(\mathbf{z}_i), \boldsymbol{\eta}_{a, i} (\mathbf{x}_{a/i}) } }} , \nonumber
\end{align}
where $\tp_{nc}^{a, i}$ is the non-conjugate part and $\veta_{a, i} (\vx_{a/i})$ is the natural parameter of the conjugate part for the factor $a$.
}

Similar to Assumption 2, we can always satisfy this assumption, but this may or may not guarantee the usefulness of our method. 

Similar to the previous section, we can reparameterize the lower bound in terms of $\vmu$ to define $\widetilde{\elbofinal}(\vmu) := \elbofinal(\vlambda)$ and then use the update \eqref{eq:prox}. Due to the mean-field approximation and linearity of the first term in \eqref{eq:prox}, we can conveniently rewrite the objective as a sum over all nodes $i$:
\begin{align}
   \max_\mu \sum_{i=1}^M \sqr{ \left\langle \vmu_i, \widehat{\nabla}_{\mu_i} \widetilde{\elbofinal} (\vmu_t) \right\rangle 
   - \frac{1}{\beta_t} \mathbb{B}_{A^*}(\vmu_i \| \vmu_{i,t}) } ,
   \label{eq:obj_mf}
\end{align}
where $\vmu_t$ and $\vmu_{i,t}$ denotes the value of $\vmu$ and $\vmu_i$, respectively, at iteration $t$.
Therefore, we can either optimize all $\vmu_i$ parallely or use a doubly-stochastic scheme by randomly picking a term in the sum.

The final algorithm is shown in Algorithm \ref{alg:cvi2} and a detailed derivation is given in Appendix \ref{app:cvi_mf}.
In Step \ref{st:vmp}, when combining all the messages received at a node $i$, the algorithm separates the conjugate computations from the non-conjugate ones.
The first set of messages $\tilde{\veta}_{a,i}$ (defined below) are obtained from the conjugate parts by taking the expectation over their natural parameters $\veta_{a,i}(\vx_{a/i})$:
\begin{align}
\tilde{\veta}_{a,i} := \myexpect_{q_{/i,t}} \sqr{ \veta_{a,i} (\vx_{a/i})} ,
\end{align}
where $q_{/i,t}$ is the variational distribution at iteration $t$ of all the nodes except $\vz_i$. By comparing the above to Equation 17 in \cite{winn2005variational}, we can see that this operation is equivalent to a message-passing step in the VMP algorithm.

The second set of messages (the second term inside the summation) in Step \ref{st:vmp} is simply the stochastic-gradient of the non-conjugate term in factor $a$. The two sets of messages are combined to get the resulting natural parameter. Finally, a convex combination is taken in Step \ref{st:update} to get the natural parameter of $q_{i,t+1}$. 

It is straightforward to see that in the absence of the second set of messages, our algorithm will reduce to VMP if we use a sequential or parallel updating scheme. 
However, an attractive property of our formulation is that we can also employ a doubly-stochastic update -- we can randomly sample terms from \eqref{eq:obj_mf}, weight them appropriately to get an unbiased stochastic approximation to the gradient, and then take a step. This will correspond to updating only a mini-batch of nodes in each iteration. 
If we use this type of updates on a conjugate-exponential model, our algorithm will be equivalent to SVI (given that we have local and global nodes and that we sample a local node followed by an update of the global node).


Convergence of Algorithm \ref{alg:cvi2} is assured under the same conditions discussed earlier.
Since the objective \eqref{eq:obj_mf} can be expressed as a sum over all the nodes, our method converges under both stochastic updates (e.g. SVI like updates) and parallel updates (e.g. with one step of VMP). 

\begin{algorithm}[t]
\caption{CVI for mean-field}
\label{alg:cvi2} 
\begin{algorithmic}[1]
   \STATE Initialize $\vlambda_{i,0}$.
    \FOR{$t=0,1,2,3,\ldots$,}
    \FOR{for all node $\vz_i$ (or a randomly sampled one)}
    \STATE $\tvlambda_{i,t} = \sum_{a\in \mathbb{N}_i} \sqr{ \tilde{\veta}_{ai} +  \widehat{\nabla}_{\mu_i} \myexpect_{q_t} (\log \tp_{nc}^{a,i}) \rvert_{\boldsymbol{\mu} = \boldsymbol{\mu}_t} }$
    \label{st:vmp}
    \STATE $\vlambda_{i,t+1} = (1-\beta_t) \vlambda_{i,t} + \beta_t\tvlambda_{i,t}$.
       \label{st:update}
       \ENDFOR
    \ENDFOR
\end{algorithmic}
\end{algorithm}

\section{Related Work}
One of the simplest method is to use local variational approximations to approximate the non-conjugate terms \citep{Jaakkola96b, Bouchard07, Khan10, WangBlei}.
Such approximations do not necessarily converge to a local maximum of the lower bound, leading to a performance loss \citep{Kuss05, marlin2011piecewise, knowles2011non,emtThesis}.
In contrast, our algorithm uses a stochastic-gradient step which is guaranteed to converge to the stationary point of the variational lower bound $\elbofinal$.

Another related approach is the Expectation-Propagation (EP) algorithm \citep{Minka01b}, which computes an exponential-family approximation (also called the site parameters) to the non-conjugate factors.
The site parameter is very similar to $\tvlambda_t$ in our algorithm, although our approximation is obtained by maximizing the lower bound unlike EP which uses moment matching.
EP suffers from numerical issues and requires damping to ensure convergence, while our method has convergence guarantees.

The Non-conjugate variational message-passing (NC-VMP) algorithm \citep{knowles2011non} is a variant of VMP for multinomial and binary regression. We can show that NC-VMP is a special case of our method under these conditions: gradients are exact and the step-size $\beta_t = 1$ (a formal proof is given in Appendix \ref{app:ncvmp}).
Therefore, our method is a stochastic version of NC-VMP with a principled way to do damping. \cite{knowles2011non} also used damping in their experiment, although it was used as a trick to make the method work.

Another related method is proposed by \cite{salimans2013fixed}. 
They view the optimality condition as an instance of a linear regression and propose a stochastic-optimization method to solve it.
This requires a stochastic estimate of the Fisher information matrix and the following gradient $\vg_\lambda := \nabla_\lambda \myexpect_{q} [\log p(\vy,\vz)]$ with respect to the natural parameter. Denoting these two quantities at iteration $t$ by $\widehat{\vC}_{\lambda,t}$ and $\widehat{\vg}_{\lambda,t}$, they do the following update: $\vlambda_{t+1} = \widehat{\vC}_{\lambda,t}^{-1} \widehat{\vg}_{\lambda,t}$. By comparing this update to \eqref{eq:grad_mu}, we can see that
the quantity in the right hand side of this update is similar to the gradient with respect to $\vmu$.
However, in their method, the two stochastic gradients are maintained and averaged over iterations.
Therefore, they need to store the Fisher information matrix explicitly which might be infeasible when the number of variational parameters is large (e.g. the GP model).
We do not have this problem, because these gradients are required only when the gradient with respect to $\vmu$ is not easy to compute directly, and can be computed on the fly at every iteration. 

Our method is closely related the two existing works by \cite{Khan15nips} and \cite{khan2016faster} which use proximal-gradient methods for variational inference.
Both of these works propose a splitting of the lower bound which is then optimized by using proximal-gradient methods in the natural-parameter space. Their update however does not directly correspond to an update in conjugate models, even though sometimes they can be obtained in closed-form.
In contrast, we propose mirror-descent \emph{without} any splitting and still obtain a closed-form update. In addition, our update corresponds to an update in a conjugate model.
The key idea behind is to optimize in the mean-parameter space.
\cite{khan2016faster} speculate that their method could be generalized to a larger class of models. Our method fills this gap and derives a generalization to exponential-family models.
Overall, our method is a sub-class of the proximal-gradient methods discussed in \cite{khan2016faster}, but it provides a simpler way of applying it to non-conjugate exponential-family models.

\section{Results}
We present results on the following four models: Bayesian logistic regression, gamma factor model, gamma matrix factorization, and Gaussian process (GP) classification.
Due to space constraints, details of the datasets are given in Appendix \ref{app:dataset}.
Additional results on Bayesian logistic regression and GP classification are in Appendix \ref{app:add}.
\begin{figure*}[t]
\center
\subfigure[Bayesian Logistic Regression with $N>D$]{\includegraphics[width = 2.1in]{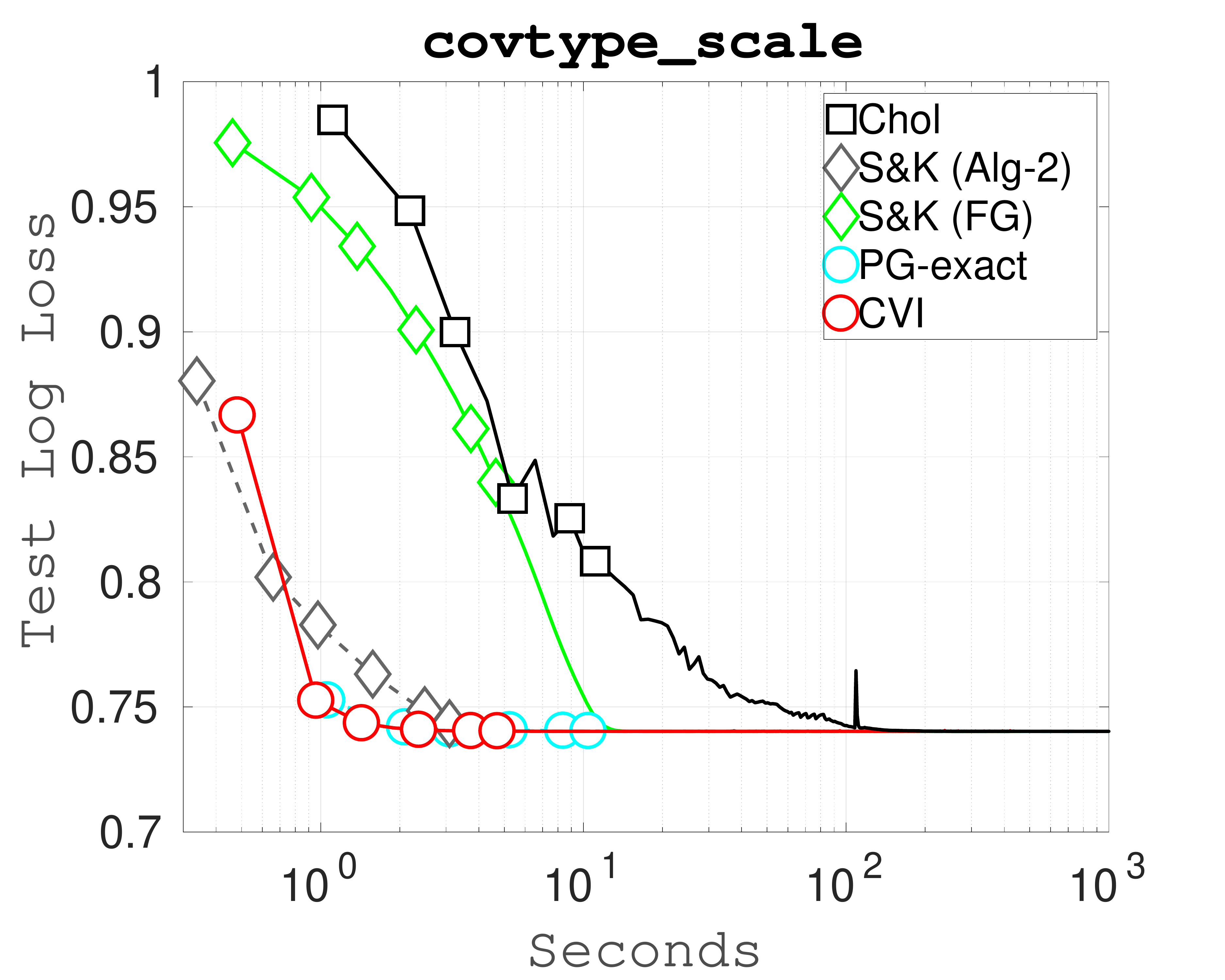}
\includegraphics[width = 2.1in]{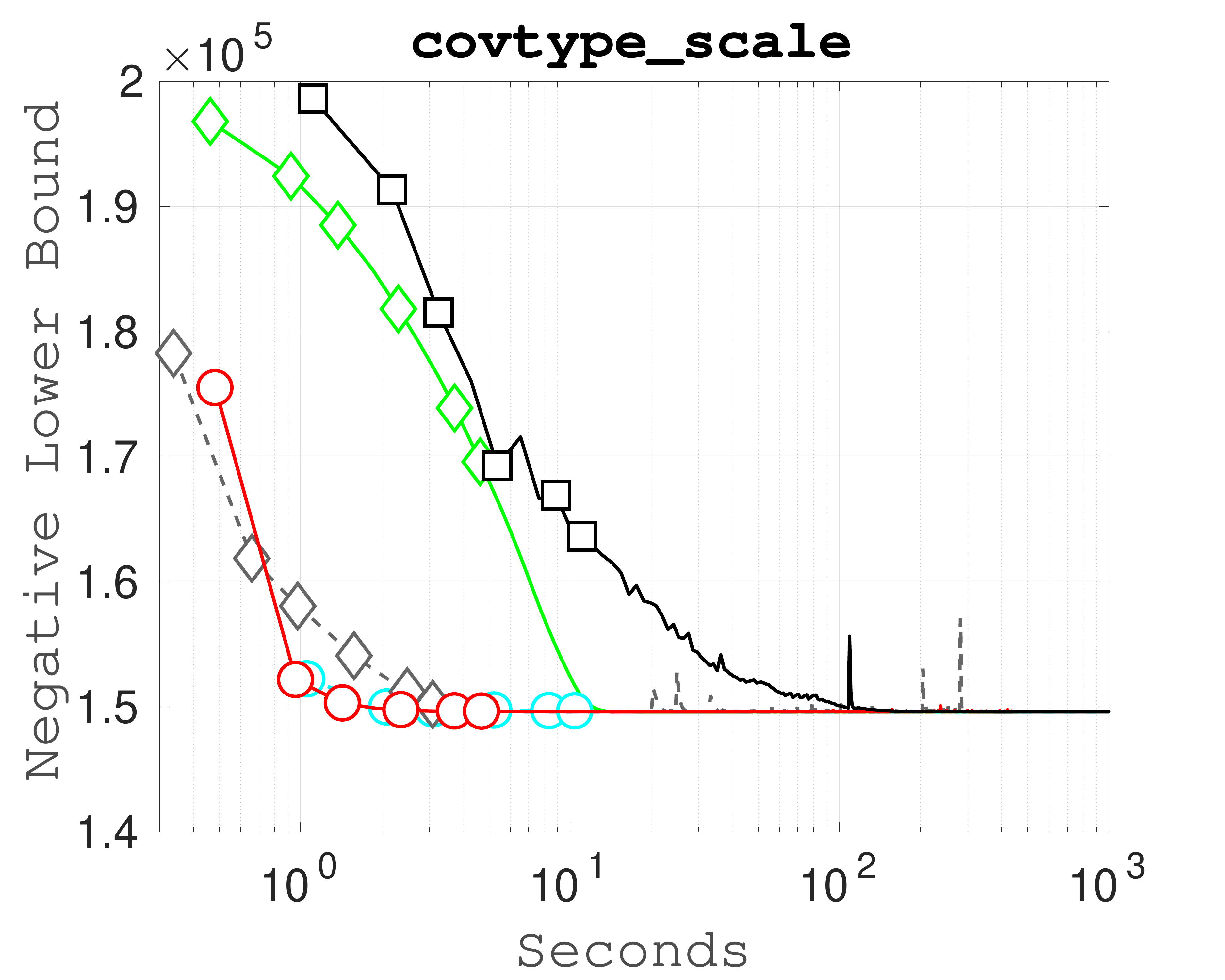}
}
\subfigure[Gamma Factor Model]{\includegraphics[width = 2.1in]{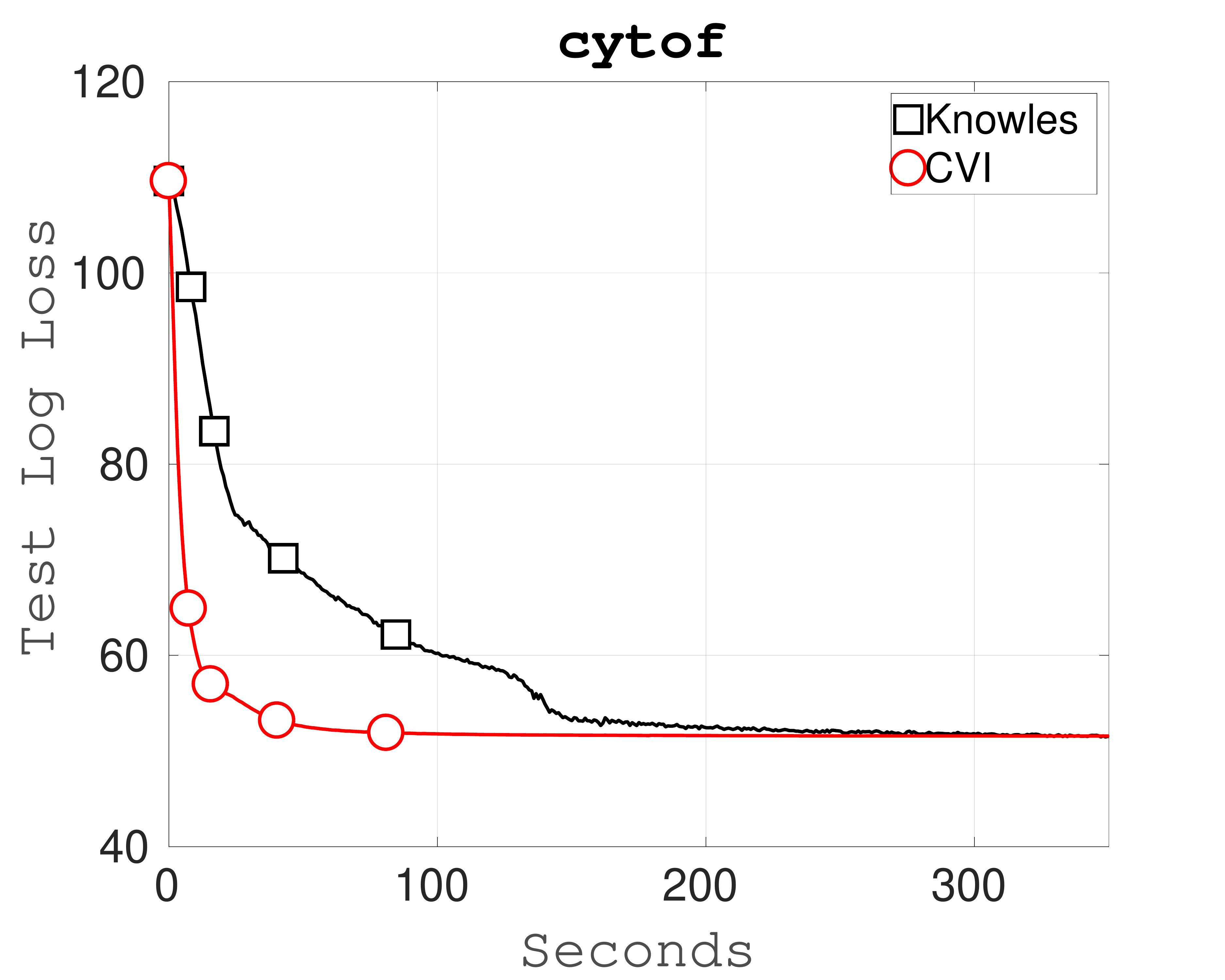}}
\vfill
\subfigure[Bayesian Logistic Regression with $D>N$]{\includegraphics[width = 2.1in]{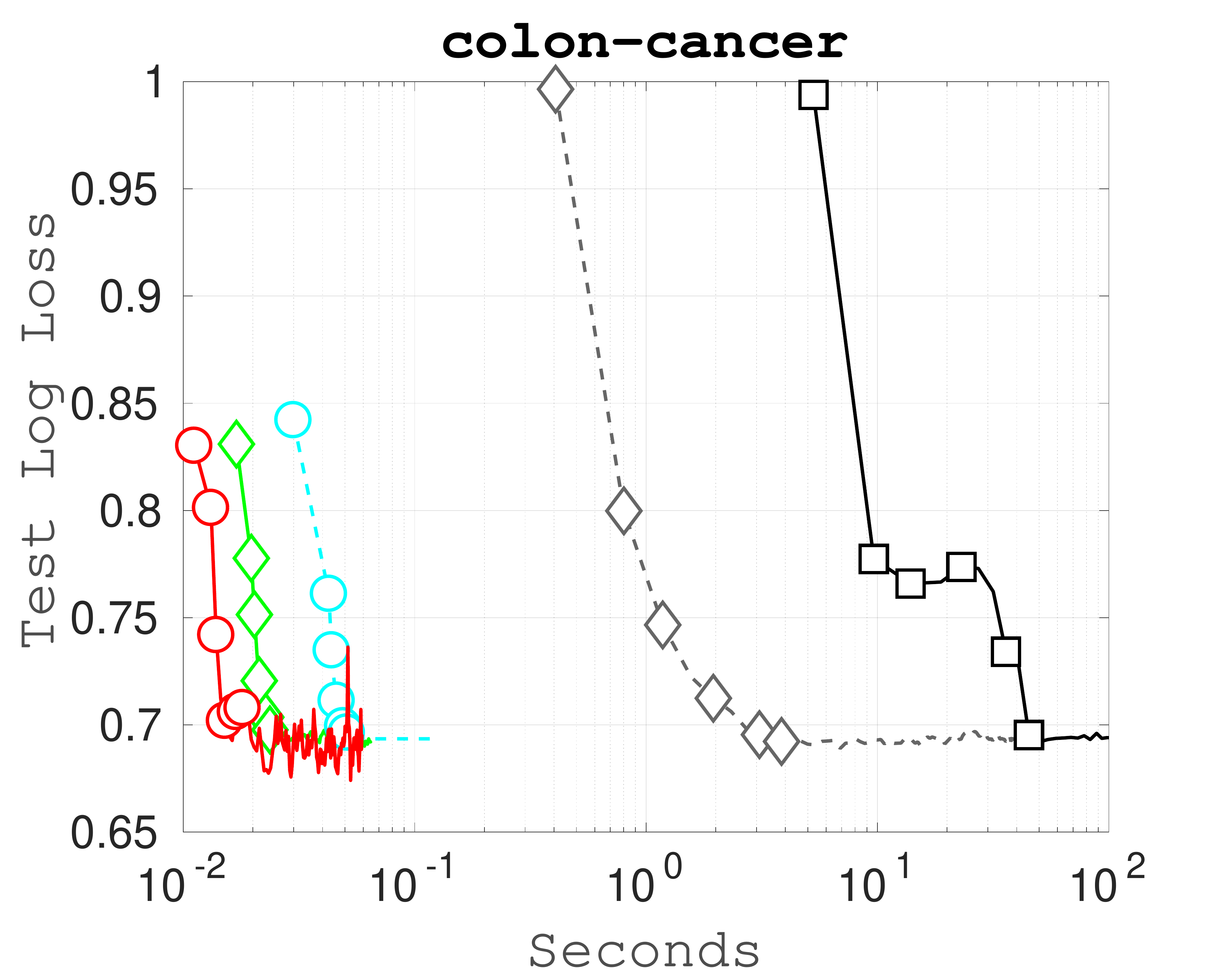} \includegraphics[width = 2.1in]{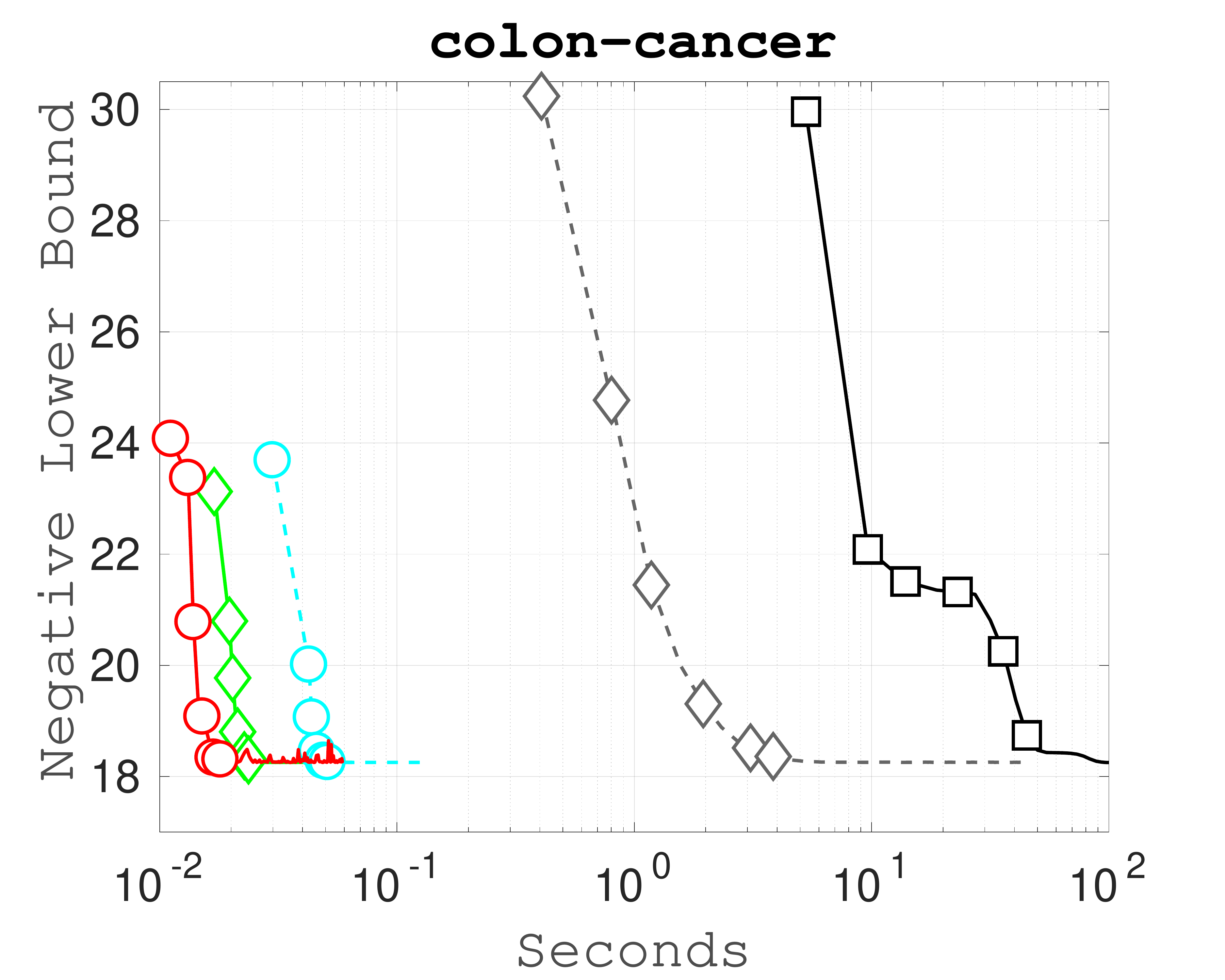}
}
\subfigure[Poisson-gamma Matrix Factorization]{\includegraphics[width = 2.1in]{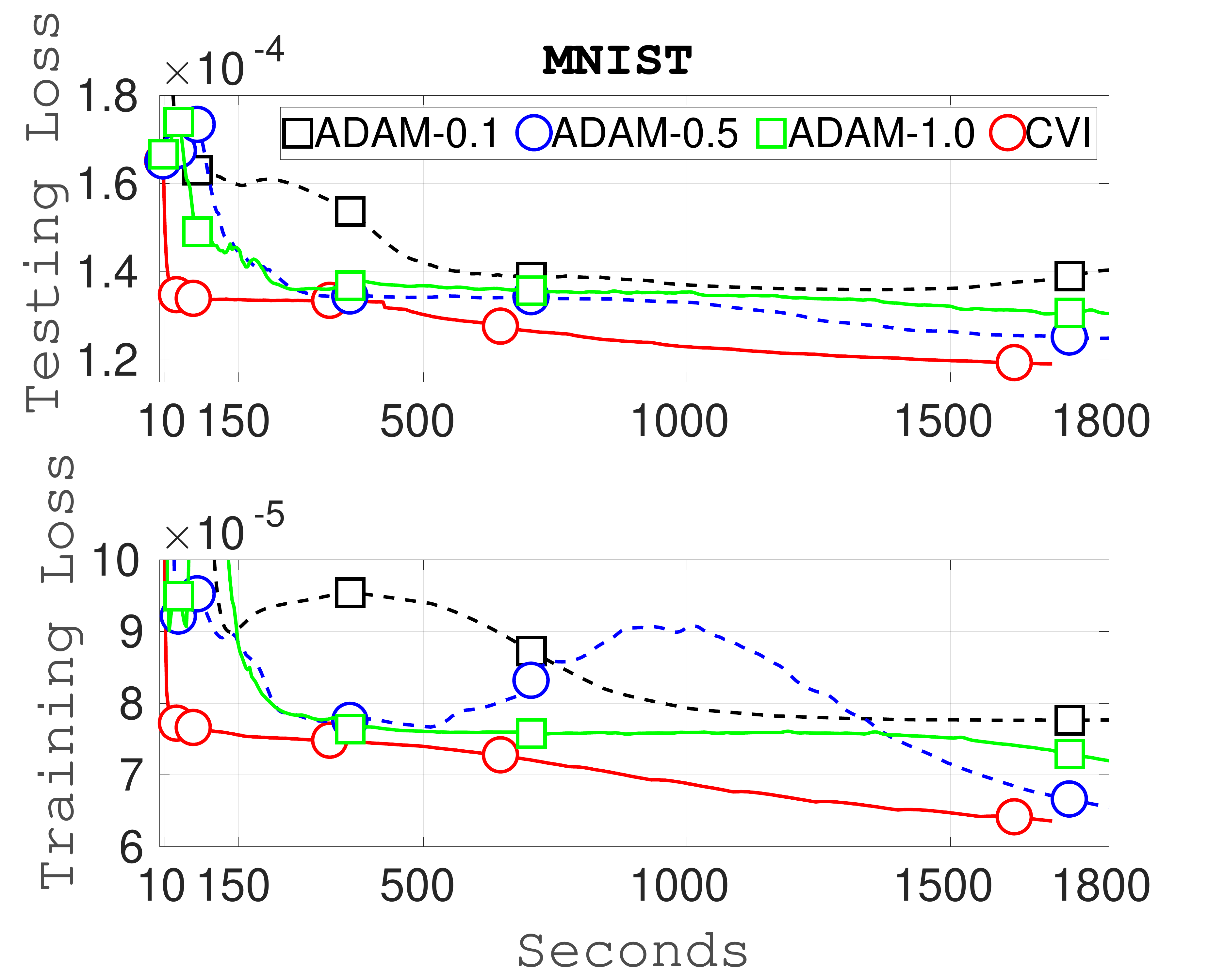}}
\caption{Comparison on Bayesian logistic regression, gamma factor model, and Poisson-gamma matrix factorization.}
\label{fig:compare}
\end{figure*}

We first discuss results for Bayesian logistic regression.
We compare our method to the following four methods: explicit optimization with LBFGS method using Cholesky factorization and exact gradients (`Chol'), proximal-gradient algorithm of \cite{Khan15nips} (PG-exact), Algorithm 2 of \cite{salimans2013fixed} (`S\&K-Alg2'), and Factorized-Gradient method of \cite{salimans2013fixed} (`S\&K-FG').

The `Chol' method does not exploit the structure of the problem and we expect it to be slow.
The `S\&K-FG' works better than `S\&K-Alg2' when $D>N$ where $N$ is the number of examples and $D$ is the dimensionality. 
This is because S\&K-Alg2 maintains an estimate of the Fisher information matrix which slows it down when $D$ is large.
However, the order is switched when $N>D$ with S\&K-Alg2 performing better than S\&K-FG. For these two methods we use the code provided by the authors.

Settings of various algorithmic parameters for these algorithms is given in Appendix \ref{app:algo_details}.
For our method, we use CVI algorithm with stochastic gradient obtained using Monte Carlo (`CVI'). Comparison to our algorithm with exact gradients is given included in Appendix \ref{app:algo_details}. Details of CVI update are given in Appendix \ref{app:glm}.

We compare the negative of the lower bound on the training set and the log-loss on the test set. The latter is computed as follows: given a test data with label $y_n \in \{0,1\}$ and the estimate of $\hat{p}_n := \myexpect_q[p(y_n =1|\vz)]$ obtained by using a method, the log-loss is equal to the negative of $ [y_n \log_2 \hat{p}_n + (1-y_n) \log_2 (1-\hat{p}_n) ] $. We report the average of the log-loss over all test points $y_n$. 
A value of 1 is equal to the performance of a random coin-toss, and any value lower than that is an improvement.

Figure (a) shows results on the `covtype' dataset with $N=581$K and $D=54$.
The markers are drawn after iteration 1, 2, 3, 5, 8, and 10.  
We use 290K examples for training and leave the rest for testing.
Chol is slowest as expected.
Since $N>D$ for this dataset, S\&K-Alg2 is faster than S\&K-FG. CVI is as fast as S\&K-Alg2 and PG-exact.

Figure (c) shows results on the Colon-Cancer dataset where $D>N$ ($N=62$ and $D=2000$). We use 31 observations for training. 
The markers are drawn after iteration 1, 2, 3, 5, 8, and 10.
We can see that the S\&K-Alg2 is much slower now, while S\&K-FG is much faster. This is because the former computes and stores an estimate of the Fisher information matrix, which slows it down.
In our approach, we completely avoid this computation and directly estimate the gradient of the mean (using Appendix \ref{sec:gp_grad}), that is why we are as fast as S\&K-FG.

Overall, we see that CVI works well in both $N>D$ and $D>N$ regimes. The main reason behing is that CVI has closed-form updates which enables the application of matrix-inversion lemma to perform fast conjugate computations (Bayesian linear regression in this case). 

Figure (b) shows the results for the gamma factor model discussed in \cite{knowles2015stochastic} (with $N=522$K and $D=40$). Details of the model are given in Appendix \ref{sec:gamma2}.
We compare CVI to the algorithm proposed by \cite{knowles2015stochastic}.
We use a constant step-size for CVI.
For stochastic gradient computations, we use the implementation provided by the author.
We use 40 latent factors and a fixed noise variance of 0.1. 
We compare the perplexity (average negative log-likelihood over test data points) using 2000 MC samples.
Markers are shown after iteration 0, 9, 19, 49, and 99.
Our method converges faster than the baseline, while achieving the same error.

Finally, Figure (d) shows the results for the Poisson-gamma matrix factorization model \citep{ranganath2015deep}. Details are in Appendix \ref{app:def}. We use the MNIST dataset of 70K images, each with 784 pixels. We use the provided train and test split, where 60K images are used for training. We use 100 latent factors and fixed model hyper-parameters.
For the baseline, we use ADAM with the following transformation $\vlambda = \log(1+\exp(\vlambda'))$ to satisfy the positivity constraint. We use many initial step-sizes for ADAM shown in Figure (d) (see `ADAM-x' where x denotes the initial step-size). For CVI algorithm, a constant step size is used. For CVI, stochastic gradients sometimes violate the constraints (because of the reparameterization trick). To deal with this, we shrink the step size such that the steps are just within the constraints  (similar to a method used in
\cite{KhanAFS13}). Computation of stochastic gradients is based on the method of \cite{knowles2015stochastic} and is the same for both methods. 
We report the following reconstruction loss: $\sqrt{\Sigma_{ij} (y_{ij} - \hat{y}_{ij} )^2 }/(V\times N)$
where $N$ denotes the number of images and $V$ is the number of pixels. 
Markers are drawn after iteration 50, 100, 500, 1000, and 2500.
Our method converges faster than the baseline and also achieves a lower error.

\section{Conclusions}
In this paper, we proposed a new algorithm called the Conjugate-computation Variation Inference (CVI) algorithm for variational inference for models that contain both conjugate and non-conjugate terms.
Our method is based on a stochastic mirror descent method in mean-parameter space. 
Each gradient step of our method can be expressed as a variational inference in a conjugate model.
This leads to computationally efficient algorithm where stochastic gradients are employed only for the non-conjugate parts of the model, while using conjugate computations for the rest.
Overall, CVI provides a general, modular, computationally efficient, scalable, and convergent approach for variational inference in non-conjugate models.
CVI is a generalization of VMP and SVI to non-conjugate models.

Our method might be useful in simplifying inference in deep generative models. For example, \cite{johnson2016composing} propose a similar algorithm for graphical models with neural-network based observation-likelihoods.
Their method does not easily generalize to models containing arbitrary conjugacy structure.
Another issue is that inference in the conditionally-conjugate part need to be run until convergence (or to a sufficient decrease in the lower bound) before updating the non-conjugate part.
Our method can be useful in fixing these issues.
Similar examples are discused in \cite{krishnan2015deep} and \cite{archer2015black} where our method can be useful for inference.

{\bf Acknowledgments:} We would like to thank the anonymous reviewers for their feedback. Part of this work was done when MEK was a scientist in EPFL (Switzerland) and WL was a freelance researcher in Toronto (Canada).

\bibliographystyle{apalike}
\bibliography{paper}

\newpage
\begin{appendix}
    \onecolumn
{\Large {\bf Supplementary Material For Conjugate-Computation Variational Inference
}}
\section{Definition of Conjugacy}
\label{conj:def}
The following definition is taken from Chapter 2 of \cite{gelman2014bayesian}. Suppose $\mathcal{F}$ is the class of data distributions $p(\vy|\vz)$ parameterized by $\vz$, and $\mathcal{P}$ is the class of prior distributions for $\vz$, then the class $\mathcal{P}$ is \emph{conjugate} for $\mathcal{F}$ if
\begin{align}
p(\vz|\vy) \in \mathcal{P}, \quad \forall p(\cdot|\vtheta) \in \mathcal{F} \,\, \textrm{and} \, \, p(\cdot) \in \mathcal{P}
\end{align}

\section{Variational Inference in the GP Model and Issues with the SGD Algorithm}
\label{sec:issueswithsgd}
   \label{sec:gp_lb}
To derive the lower bound we substitute the joint-distribution \eqref{eq:gp_joint} in the lower bound \eqref{eq:lb} and simplify:
\begin{align}
   \elbofinal(q) &:= \myexpect_q [\log p(\vy,\vz) - \log q(\vz)] \\
                 &= \myexpect_q \sqr{ \sum_{n=1}^N  \log p(y_n|z_n) + \log \gauss(\vz|0,\vK) - \log \gauss(\vz|\vm,\vV)} \\
                 &= \sum_n \myexpect_q\sqr {\log p(y_n|z_n) } - \dkls{}{\gauss(\vz|\vm,\vV)}{\gauss(\vz|0,\vK)} \\
                 &= \sum_n \myexpect_q\sqr {\log p(y_n|z_n) } + \half \sqr{\log|\vV| - \trace\rnd{\vK^{-1}\vV} - \vm^T\vK^{-1}\vm } + \textrm{constant}
\end{align}
We can see the special structure of the lower bound. The first term here might be intractable, but the second term (the KL divergence term) and its gradients have a closed-form expression when $q$ is Gaussian. Therefore we do not need stochastic-gradient approximations for this term. A naive SGD implementation might ignore this. 

There are at least three alternate parameterizations of the posterior $\gauss(\vz|\vm,\vV)$ in this case. We could use the natural parameters $\{ \vV^{-1}\vm, -\half\vV^{-1} \}$, or the mean parameters $\{\vm, \vV + \vm\vm^T\}$, or simply use $\{\vm,\vV\}$ itself.
   Different parameterization lead to different updates whose computational efficiency differ drastically. 
   For example, if we choose to update the inverse of covariance $\vV^{-1}$, we get the following updates:
   \begin{align}
      \vV_{t+1}^{-1} &= \vV_t^{-1} + \frac{\rho_t}{2} \sqr{ \left. \deriv{}{\vV^{-1}} \sum_n \myexpect_q\sqr {\log p(y_n|z_n) } \right\rvert_{V = V_t} - \half \vV_t + \half \vV_t \vK^{-1}\vV_t} 
      \label{eq:sgd_gp_1}
   \end{align}
   On the other hand, if we choose to update the covariance $\vV$ instead, we get the following update:
   \begin{align}
      \vV_{t+1} &= \vV_t + \frac{\rho_t}{2} \sqr{ \left. \deriv{}{\vV} \sum_n \myexpect_q\sqr {\log p(y_n|z_n) } \right\rvert_{V = V_t} + \half \vV_t^{-1} - \half \vK^{-1}} 
      \label{eq:sgd_gp_2}
   \end{align}

   The two updates are quite different. The second update involves less computation than the first one because the last term in the first update involves multiplication of three matrices. 
   Both of these steps require explicitly forming the matrix $\vV$ and $\vV^{-1}$, which might be infeasible for large $N$ (e.g. a million data points).
   In addition, they both compute inverse of $\vK$ which might be ill-conditioned.

   The above parameterization requires $O(N^2)$ memory, however, it is well known that for the GP model, there are only $O(N)$ free parameters \citep{Opper:09}.
   Choosing any of the three parameterizations discussed earlier will lead to an algorithm that is an order of magnitude slower than the best option.

   Our CVI method completely avoids this re-parameterization issue by expressing the gradient steps as a conjugate computation step.
   Our updates naturally only have $O(N)$ free variational parameter which are obtained by using stochastic-gradients of the non-conjugate terms $\myexpect_q[\log p(y_n|z_n)]$. We can reduce the number of gradients to be computed in each iteration $t$ to $O(1)$ by using a doubly-stochastic scheme. 

	\subsection{Stochastic Gradients with respect to the Mean Parameters}
\label{sec:gp_grad}
In this section, we explain the computation of the gradient of $f_n = \myexpect_q [\log p(y_n|z_n)]$ with respect to the following mean parameter of the Gaussian distribution $q(z_n) = \gauss(z_n|m_n,V_{nn})$:
\begin{align}
\mu_n^{(1)} := m_n, \quad \mu_n^{(2)} := V_{nn} + m_n^2
\end{align}

According to \citep{Opper:09}, the gradient with respect to the mean, $m_n$ and the variance, $V_{nn}$ are:
\begin{align}
   \frac{\partial f_n}{\partial m_n} = \myexpect_q \sqr{ \frac{\partial f_n }{\partial z_n} }, \quad \frac{\partial f_n}{\partial V_{nn}} = \half \myexpect_q\sqr{ \frac{\partial^2 f_n }{\partial z_n^2} }
\end{align}
Therefore, we can easily approximate these gradients by using the Monte Carlo method.
By using the chain rule, we can express the gradient with respect to the mean parameters in terms of the gradients with respect to $m_n$ and $V_{nn}$ and then use Monte Carlo. We derive these expressions below.

For notation simplicity, we drop $n$ from now and refer to $m_n$ and $V_{nn}$ as $m$ and $v$, respectively. 
We first express $m$ and $v$ in terms of the mean parameters: $m = \mu^{(1)}$ and $v = \mu^{(2)} - (\mu^{(1)})^2$.
By using the chain rule, we express the gradient with respect to the mean parameters in terms of the gradients with respect to $m$ and $v$:
\begin{align}
\frac{\partial f}{\partial \mu^{(1)}} & = \frac{\partial f}{\partial m}\frac{\partial m}{\partial \mu^{(1)}} + \frac{\partial f}{\partial v}\frac{\partial v}{\partial \mu^{(1)}} = \frac{\partial f}{\partial m} - 2 \frac{\partial f}{\partial v} m \\
\frac{\partial f}{\partial \mu^{(2)}} & = \frac{\partial f}{\partial m}\frac{\partial m}{\partial \mu^{(2)}} + \frac{\partial f}{\partial v}\frac{\partial v}{\partial \mu^{(2)}} = \frac{\partial f}{\partial v}  
\end{align}

\section{Basics of Exponential Families}
\label{app:basics}
We summarize a few results regarding exponential family.
Details of these results can be found in Chapter 3 of \cite{WainwrightJordan08}.
We assume that $q(\vz|\vlambda)$ takes the following exponential form:
\begin{align}
   q(\vz|\vlambda) = h(\vz) \exp\crl{\ang{\vlambda, \vphi(\vz)} - A(\vlambda)}
\end{align}
where $\vphi := [\phi_1,\phi_2,\ldots,\phi_M]$ is a vector of sufficient statistics,
$\vlambda := [\lambda_1,\lambda_2,\ldots,\lambda_M]^T$ is a vector of natural parameters, 
$\ang{\va,\vb}$ is an inner product,
and $A(\vlambda)$ is the log-partition function.
The set of natural parameters is denotes by $\Omega := \{\vlambda\in\real^M | A(\vlambda) < \infty\}$.

We call the above representation \emph{minimal} when there does not exist a nonzero vector $\va\in\real^M$ such that the linear combination $\ang{\va,\vphi}$ is equal to a constant.
Minimal representation implies that each distribution $q(\vz|\vlambda)$ has a unique natural parametrization $\vlambda$.

We define the mean parameter associated with a sufficient statistic $\phi_m$ as follows:
\begin{align}
\mu_m := \myexpect_q \sqr{\phi_m(\vz)}
\end{align}
We denote the vector of parameter by $\vmu$.
The set of valid mean parameters is defined as shown below:
\begin{align}
\mathcal{M} := \{\vmu \in \real^M | \exists \,\, p \,\,s.t. \,\,\myexpect_q [\phi_m(\vz)] = \mu_m, \, \forall m \}
\end{align}

It is easy to show that $A(\vlambda)$ is convex, and the mean parameter can be obtained by simply differentiating it, i.e., $\vmu = \nabla A(\vlambda)$.
The mapping $\nabla A$ is one-to-one and onto iff the representation is minimal. 
This property allows us to switch back and forth between $\Omega$ and $\mathcal{M}$. 

Since $\nabla A$ is convex, we can find its convex conjugate as follows:
\begin{align}
A^*(\vmu) := \sup_{\lambda \in \Omega} \crl{\ang{\vmu,\vlambda} - A(\vlambda) }
\end{align}
It is easy to see that $\vlambda = \nabla A^*(\vmu)$, therefore the pair of operators $(\nabla A, \nabla A^*)$ lets us switch back and forth between $\Omega$ and $\mathcal{M}$.

Bregman divergences associated with functions $A$ and $A^*$ is defined as follows:
\begin{align}
   \mathbb{B}_{A}(\vlambda_1\|\vlambda_2) &:= A(\vlambda_1) - A(\vlambda_2) - \ang{\vlambda_1 - \vlambda_2, \nabla_\lambda A(\vlambda_2)} \\
   \mathbb{B}_{A^*}(\vmu_1\|\vmu_2) &:= A^*(\vmu_1) - A^*(\vmu_2) - \ang{\vmu_1 - \vmu_2, \nabla_\mu A^*(\vmu_2)}
\end{align}

\section{Proof of Theorem \ref{thm:main}}
\label{app:proof_thm1}

To simplify the notation, we will refer to $\tp_c(\vy,\vz)$ and $\tp_{nc}(\vy,\vz)$ by simply $\tp_c$ and $\tp_{nc}$ respectively. Similarly, we will refer to $q(\vz|\vlambda)$ and $q(\vz|\vlambda_t)$ by $q$ and $q_t$ respectively.
Using this notation and the split of the joint distribution given in Assumption 2, the variational lower bound can be written as follows:
\begin{equation}
\begin{split}
   \widetilde{\elbofinal}(\vmu) = \elbofinal(\vlambda) =
   \myexpect_q[ \log \tp_{nc}] + \myexpect_q [\log (\tp_c/q) ]
    \label{eq:elbo}
\end{split}
\end{equation}
We prove Theorem \ref{thm:main} by proving several lemmas.
We start with the following lemma which shows that the linear approximation of the second term (the conjugate part) in \eqref{eq:elbo} simplifies to the term itself plus a KL divergence term.
\begin{lemma}
   \label{lemma:grad_conj}
   For the conjugate part of the lower bound, we have the following property:
   \begin{align}
      \ang{ \vmu, \nabla_{\mu} \myexpect_q [\log (\tp_c/q) ] \rvert_{\mu=\mu_t} } &= \myexpect_q \sqr{ \log (\tp_{c}/q)} + \myexpect_q \sqr{ \log (q/q_t)} + c
   \end{align}
   where $c$ is a constant that does not depend on $\vmu$ (or $\vlambda$).
\end{lemma}
\begin{proof}
By substituting the definitions of $\tp_c$ and $q$, we get the following:
\begin{align}
\ang{ \vmu, \nabla_{\mu} \myexpect_q [\log (\tp_c/q) ] \rvert_{\mu=\mu_t} }
&= \ang{ \vmu, \nabla_{\mu} \myexpect_q [\ang{\vphi(\vz), \veta-\vlambda} + A(\vlambda)] \rvert_{\mu=\mu_t} } = \ang{ \vmu, \nabla_{\mu} [\ang{\vmu, \veta-\vlambda} + A(\vlambda)] \rvert_{\mu=\mu_t} } 
\end{align}
We derive the gradient w.r.t. $\vmu$ below:
\begin{align}
\nabla_\mu \sqr{ \ang{\vmu, \veta - \vlambda} + A(\vlambda)} =
 \veta - \vlambda - \ang{\vmu, \nabla_\mu \vlambda} + \nabla_\mu A(\vlambda) =
\veta - \vlambda - \vC_\lambda^{-1} \vmu+ \vC_\lambda^{-1}\vmu = \veta - \vlambda
\label{eq:grad_mu_1}
\end{align}
where $\vC_\lambda$ is the Fisher-information matrix and we use the fact that the gradient w.r.t. $\vmu$ is equal to $\vC_\lambda^{-1}$ times the gradient w.r.t. $\vlambda$ (this is explained in Appendix \ref{app:gradmu}.
Substituting this back,
\begin{align}
\ang{ \vmu, \nabla_{\mu} \myexpect_q [\log (\tp_c/q) ] \rvert_{\mu=\mu_t} } &= \ang{ \vmu, \veta - \vlambda_t } \\
& = \myexpect_q \sqr{ \ang{\vphi(\vz), \veta - \vlambda_t} + A(\vlambda_t)  } + c \\
&= \myexpect_q \sqr{ \log (\tp_{c}/q_t)} + c \\
&= \myexpect_q \sqr{ \log (\tp_{c}/q)} + \myexpect_q \sqr{ \log (q/q_t)} + c
\end{align}
\end{proof}
The following lemma shows that the Bregman divergence is equal to the KL divergence which has a convenient form.
\begin{lemma}
   \label{lemma:bregman}
   For all $q$ and $q_t$ satisfying Assumption 1, we have the following relationships:
   \begin{align}
      \mathbb{B}_{A^*}(\vmu\|\vmu_t) = \mathbb{B}_{A}(\vlambda_t\|\vlambda) = \myexpect_q[\log (q/q_t)]
   \end{align}
\end{lemma}
\begin{proof}
The following equivalence holds between the two Bregman divergences defined using $A$ and $A^*$ (see \cite{raskutti2015information}, for example): $\mathbb{B}_A(\vlambda_t \| \vlambda) = \mathbb{B}_{A^*}(\vmu \| \vmu_t) $. The last equality can be proved as follows:
\begin{align}
\myexpect_q [\log(q/q_t)] &= \myexpect_q \sqr{ \ang{\vphi(\vz), \vlambda} - A(\vlambda) - \ang{\vphi(\vz), \vlambda_t} + A(\vlambda_t) } \\
&= A(\vlambda_t) - A(\vlambda) - \ang{\vlambda_t - \vlambda, \nabla A(\vlambda)}  \\
&= \mathbb{B}_A(\vlambda_t \| \vlambda) = \mathbb{B}_{A^*}(\vmu \| \vmu_t)
\end{align}
\end{proof}
Denoting the gradient of the non-conjugate term by $\vg_t := \widehat{\nabla}_\mu \myexpect_q[ \log \tp_{nc}] \rvert_{\mu=\mu_t}$, the following lemma shows that using Lemma \ref{lemma:grad_conj} and \ref{lemma:bregman} we can get a closed-form solution for \eqref{eq:prox}.
\begin{lemma}
   \label{lemma:post_exp}
   The solution of \eqref{eq:prox} is equal to the mean $\vmu_{t+1}$ of the following distribution: 
   \begin{align}
      q_{t+1} \propto \crl{ e^{\ang{\boldsymbol{\phi}(\mathbf{z}), \mathbf{g}_t}} \tp_c }^{\beta_t} (q_t)^{1-\beta_t}
   \end{align}
\end{lemma}
\begin{proof}
Using \eqref{eq:elbo}, we get the following expression for the first term in \eqref{eq:prox} which we simplify in the second line using Lemma 1:
\begin{align}
\ang{\vmu, \widehat{\nabla}_\mu \widetilde{\elbofinal}(\vmu_t) } &= \ang{\vmu, \widehat{\nabla}_\mu \myexpect_q [\log \tp_{nc}] + \widehat{\nabla}_\mu \myexpect_q [\log (\tp_c/q)] } \\ 
&= \ang{\vmu, \widehat{\nabla}_\mu \myexpect_q [\log \tp_{nc}]} + \myexpect_q [\log (\tp_c/q)]  + \myexpect_q [\log (q/q_t)] + c 
\end{align}
Plugging this in \eqref{eq:prox} and using Lemma 2, we get the following objective function:
\begin{align}
&\ang{\vmu, \widehat{\nabla}_\mu \myexpect_q [\log \tp_{nc}]} + \myexpect_q [\log (\tp_c/q)]  + \myexpect_q [\log (q/q_t)] - \frac{1}{\beta_t} \myexpect_q [\log (q/q_t)] \\
&= \ang{\vmu, \widehat{\nabla}_\mu \myexpect_q [\log \tp_{nc}]} + \myexpect_q [\log (\tp_c/q)]  - \frac{1-\beta_t}{\beta_t} \myexpect_q [\log (q/q_t)] \\
&= \myexpect_q \sqr{ \ang{\vphi(\vz), \widehat{\nabla}_\mu \myexpect_q [\log \tp_{nc}]} + \log (\tp_c/q)  - \frac{1-\beta_t}{\beta_t} \log (q/q_t) } \\
&= \myexpect_q \sqr{ \log \frac{\exp\crl{ \ang{\vphi(\vz), \widehat{\nabla}_\mu \myexpect_q [\log \tp_{nc}]} }  \tp_c  q_t^{(1-\beta_t)/\beta_t}}{q^{1+ (1-\beta_t)/\beta_t}} } \\
&= \myexpect_q \sqr{ \log \frac{\exp\crl{ \ang{\vphi(\vz), \widehat{\nabla}_\mu \myexpect_q [\log \tp_{nc}]} } \tp_c  q_t^{(1-\beta_t)/\beta_t}}{q^{1/\beta_t}} } \\
&= \frac{1}{\beta_t} \myexpect_q \sqr{ \log \frac{\rnd{ \exp\crl{ \ang{\vphi(\vz), \widehat{\nabla}_\mu \myexpect_q [\log \tp_{nc}]}  } \tp_c }^{\beta_t} q_t^{(1-\beta_t)}}{q} } 
\end{align}
The numerator is an unnormalized exponential family distribution which takes the same exponential family form as $q$ (note that the base measure $h(\vz)$ is present in both $\tp_c$ and $q_t$ which sums to $h(\vz)$ due to convex combination). The normalizing constant of this distribution does not depend on $\vmu$, therefore the minimum is obtained when the numerator is equal to the denominator (minimum of the KL divergence). This proves the lemma.
\end{proof}
Finally, the following lemma uses recursion to express the solution as a Bayesian inference in a conjugate model.
\begin{lemma}
   \label{lemma:final}
   Given the conditions of Theorem \eqref{thm:main}, the distribution $q_{t+1}$ is equal to the posterior distribution of the following model: $q_{t+1} \propto \exp(\ang{\boldsymbol{\phi}(\mathbf{z}), \widetilde{\boldsymbol{\lambda}}_t } )\,  \tp_c $. 
\end{lemma}
\begin{proof}
Denote the gradient of the non-conjugate term by $\vg_t := \widehat{\nabla}_\mu \myexpect_q[ \log \tp_{nc}] \rvert_{\mu=\mu_t}$,
If we initialize $q_1 \propto \tp_c$ and $\tvlambda_0 := 0$, we can apply recursion to express $q_{t+1}$ as a conjugate model. We demonstrate this for $q_1,q_2,$ and $q_3$ below:
\begin{align}
q_1 &\propto (\tp_c)^{\beta_0} (\tp_c)^{1-\beta_0} = \tp_c\\
q_2 &\propto \exp{\ang{\vphi(\vz), \, \beta_1 \vg_1}} (\tp_c)^{\beta_1} (q_1)^{1-\beta_1} \\
&= \exp{\ang{\vphi(\vz), \, \beta_1 \vg_1}} (\tp_c)^{\beta_1} \tp_c^{1-\beta_1} \nonumber\\
&= \exp{\ang{\vphi(\vz), \, \beta_1\vg_1 }} \tp_c \\
&= \exp{\ang{\vphi(\vz), \, \tvlambda_1 }} \tp_c \\
q_3 &\propto \exp{\ang{\vphi(\vz), \, \beta_2 \vg_2}} (\tp_c)^{\beta_2} (q_2)^{1-\beta_2} \\
&= \exp{\ang{\vphi(\vz), \, \beta_2 \vg_2 + (1-\beta_2) \tvlambda_1 }} \tp_c \\
&= \exp{\ang{\vphi(\vz), \, \tvlambda_2 }} \tp_c 
\end{align}
Proceeding as above, we get the required result.

\end{proof}

\section{Examples of CVI}
\label{app:ex_cvi}
\subsection{Example: Generalized Linear Model}
\label{app:glm}
A GLM assumes the following joint distribution:
  \begin{align}
  p(\vy,\vz) &:= 
  \underbrace{\sqr{\prod_{n=1}^N p(y_n|\tvx_n^T\vz)}}_{\tilde{p}_{nc}(\mathbf{y},\mathbf{z})} \gauss(\vz|0,\delta\vI)
\end{align}
where $\widetilde{\vx}_n = [1, \vx_n^T]^T$.
For a Gaussian distribution $q := \gauss(\vz|\vm,\vV)$, the data terms $p(y_n|\tvx_n^T\vz)$ are the non-conjugate terms. We define $\eta_n := \tvx_n^T\vz$ and use its mean parameters in a similar way as GPs to obtain the natural parameter approximations $\tlambda_{n,t}^{(1)}$ and $\tlambda_{n,t}^{(2)}$ of the data term $p(y_n|\tvx_n^T\vz)$. 
In step \ref{st:project} of Algorithm \ref{alg:cvi1}, these are updated as follows:
\begin{align}
\tlambda_{n,t}^{(i)} = (1-\beta_t) \tlambda_{n,t-1}^{(i)} + \beta_t \widehat{\nabla}_{\mu_n^{(i)}} \myexpect_{q_t} [\log p(y_n|\eta_n) ] \rvert_{\mu = \mu_t}   \nonumber
\end{align}
where $\mu_n^{(i)}$ is the $i$'th mean parameter of $q(\eta_n)$.

Using the above parameters for the approximations, we can write Step \ref{st:infer} as a conjugate computation in the following Bayesian linear regression:
\begin{align}
    q_{t+1} \propto \sqr{\prod_{n=1}^N \gauss(\tilde{y}_{n,t}|\tvx_n^T \vz, \tilde{\sigma}_{n,t}^2) } \gauss(\vz|0,\delta\vI) \nonumber 
\end{align}
where $\tilde{y}_{n,t} = \tilde{\sigma}_{n,t}^2 \lambda_{n,t}^{(1)}$, $\tilde{\sigma}_{n,t}^2 = -1/(2\tlambda_{n,t}^{(2)})$,

\subsection{Example: Kalman Filters with GLM Likelihoods}
\label{app:kalman}
We seek a Gaussian approximation $q(\vz) = \gauss(\vz|\vm,\vV)$ to the following time-series model (we denote time by $k$ to differentiate it from the iteration $t$):
\begin{align}
   p(\vy,\vz) = \gauss(z_0|0, 1) \prod_{k=1}^K \gauss(z_k|z_{k-1}, \sigma^2) \underbrace{\prod_{k=1}^K p(y_k|z_k)}_{\tp_{nc}(\mathbf{y},\mathbf{z})}
\end{align}
The likelihood terms $p(y_k|z_k)$ are non-conjugate to $q$ and by using our method we can approximate them by $\gauss(\tilde{y}_{k,t}|z_{k-1}, \tilde{\sigma}^2_{k,t})$
where $\tilde{y}_{k,t} = \tilde{\sigma}_{k,t}^2 \lambda_{k,t}^{(1)}$, $\tilde{\sigma}_{k,t}^2 = -1/(2\tlambda_{k,t}^{(2)})$, and $\tlambda_{k,t}^{(i)}$ are updated as follows:
\begin{align}
   \tlambda_{k,t}^{(i)} = (1-\beta_t) \tlambda_{k,t-1}^{(i)} + \beta_t \widehat{\nabla}_{\tilde{\mu}_k^{(i)}} \myexpect_{q_t} [\log p(y_k|z_k) ] \rvert_{\mu=\mu_t}   \nonumber
\end{align}
with $\tilde{\mu}_k^{(i)}$ being the $i$'th mean parameter of $q(z_k)$.

\subsection{Example: A Gamma Distribution Model}
\label{app:gamma1}
We consider a simple non-conjugate Gamma distribution model discussed by \cite{knowles2012bayesian}.
We use the following definition of the Gamma distribution:
$\Ga(x|\alpha,\beta) \propto x^{\alpha-1} e^{-x\beta}$,
where $x, \alpha$, and $\beta$ are all non-negative scalars.

Given a Gamma distributed scalar observation $y$, we place a Gamma prior on the shape parameter $z$, as shown below:
\begin{align}
p(y,z) = \underbrace{\Ga(y|z,1)}_{\tp_{nc}(y,z)} \Ga(z|a,b)
\end{align}
The rate of the likelihood is fixed to 1, and the prior parameters $a$ and $b$ are known.
Our goal is to find the posterior distribution $p(z|y)$ which we will approximate with a Gamma distribution: $q(z) = \Ga(z|\alpha,\beta)$.
Clearly, the likelihood is non-conjugate to $q$.

The sufficient statistics and mean parameters of a Gamma distribution are as follows:
\begin{align}
\phi_1(z) = z, &\,\, \mu_1 := \myexpect_q[ \phi_1(z) ] = \psi(\alpha) - \log \beta \nonumber \\
\phi_2(z) = \log z, &\,\, \mu_2 := \myexpect_q[ \phi_2(z) ] = \alpha/\beta  
\end{align}
where $\psi$ is the digamma function. 
Using these in the CVI updates we get the following update:
\begin{align}
    q_{t+1} &\propto e^{\sqr{ z \tlambda_{t}^{(1)} + (\log z) \tlambda_{t}^{(2)} }} \Ga(z|a,b) 
\end{align}
where $\tlambda_{t}^{(i)}$ are updated as follows for $i=1,2$:
\begin{align}
   \tlambda_{t}^{(i)} &= (1-\beta_t) \tlambda_{t-1}^{(i)}+ \beta_t \widehat{\nabla}_{\tilde{\mu}_i} \myexpect_{q_t} [\log p(y|z) ] \rvert_{\mu=\mu_t} \label{eq:gamma_update}
\end{align}
The approximated term is conjugate to the Gamma distribution and therefore it is straightforward to compute the posterior parameters.

\section{Gradient with respect to $\vmu$ for exponential family}
\label{app:gradmu}
For some distributions in the exponential family, it may be difficult to directly compute the gradient with respect to $\vmu$. 
We propose to express the gradient w.r.t. $\vmu$ in terms of the Fisher information matrix and the gradient w.r.t. the natural parameter, by using the chain rule.
Given the following function of interest $f(\vmu)=\myexpect_{q(\vz)}\sqr{ h(\vz) }$, we can formally express this as follows:
\begin{align}
   \deriv{f}{\vmu} &= \sqr{ \secondderiv{A(\vlambda)}{\vlambda} }^{-1} \deriv{f}{\vlambda} 
\end{align}
Since each of these quantities can be written as expectations, as shown below, we can use the re-parametrization trick \cite{kingma2013auto} along with the Monte Carlo method to approximate them.
\begin{align}
{ \secondderiv{A(\vlambda)}{\vlambda} } &= \frac{\partial \vmu}{\partial \vlambda} = \frac{\partial \myexpect_{q(\vz)}[\vphi(\vz)]}{\partial \vlambda} \\
\frac{\partial f}{\partial \vlambda} & = \frac{\partial \myexpect_{q(\vz)}[h(\vz)]}{\partial \vlambda} 
 \end{align}
where  $\vphi(\vz)$ is the sufficient statistics of $q(\vz)$.

\section{Derivation of the CVI Algorithm for Mean-Field}
\label{app:cvi_mf}
We rewrite the objective function which naturally splits over $i$:
\begin{align}
   \max_\mu \sum_{i=1}^M \sqr{ \left\langle \vmu_i, \widehat{\nabla}_{\mu_i} \widetilde{\elbofinal} (\vmu_t) \right\rangle 
   - \frac{1}{\beta_t} \mathbb{B}_{A^*}(\vmu_i \| \vmu_{i,t}) }
\end{align}
We can optimize each $\vmu_i$ parallely or use a doubly-stochastic method to optimize.

In the following, $\vmu_{/i}$ denotes the mean-parameter vector without $\vmu_i$.

To optimize with respect to a $\vmu_i$, we need to express the lower bound as a function of $\vmu_i$.
By using Assumption 4, the lower bound with respect to $\vmu_i$ can be expressed as a sum over non-conjugate and conjugate parts. We show this below in \eqref{eq:uu1} which is obtained by replacing the joint distribution by the conditional of $\vz_i$. The second step afterwards is obtained by substituting \eqref{eq:ass4} from Assumption 4.
The third step is obtained by using the definition of $q_i(\vz_i|\vlambda_i)$ given in \eqref{eq:qi} in Assumption 3.
The fourth step is obtained by taking the expectation inside.
\begin{align}
\widetilde{L}(\vmu_i,\vmu_{/i}) &= \myexpect_q \sqr{\log p(\vz_i|\vx_{/i}) - \log q_i(\vz_i|\vlambda_i) } + \textrm{constant} \label{eq:uu1}\\  
&= \myexpect_q \sqr{ \log h_i(\vz_i) + \sum_{a\in\mathbb{N}_i} \log \tp_{nc}^{a,i}(\vz_i,\vx_{a/i}) + \sum_{a\in\mathbb{N}_i} \ang{\vphi_i(\vz_i), \veta_{a,i} (\vx_{a/i})} -\log q_i(\vz_i|\vlambda_i) }  + \textrm{constant} \\
&= \myexpect_q \sqr{ \sum_{a\in\mathbb{N}_i} \log \tp_{nc}^{a,i}(\vz_i,\vx_{a/i}) + \sum_{a\in\mathbb{N}_i} \ang{\vphi_i(\vz_i), \veta_{a, i} (\vx_{a/i}) - \vlambda_i} + A_i(\vlambda_i) }  + \textrm{constant} \\
&= \sum_{a\in\mathbb{N}_i} \myexpect_{q} \{ \log \tp_{nc}^{a, i}(\vz_i,\vx_{a/i}) \} + \ang{\vmu_i, \sum_{a\in\mathbb{N}_i} \myexpect_{q_{/i}} \{ \veta_{a, i} (\vx_{a/i}) \} - \vlambda_i} + A_i(\vlambda_i)  + \textrm{constant} 
\end{align}

This is similar to \eqref{eq:elbo} since the first term is non-conjugate while the rest of the terms correspond to conjugate parts in the model. We rewrite this below by using the notation $\tilde{\veta}_{a,i} := \myexpect_{q_{/i,t}} \{ \veta_{a,i} (\vx_{a/i}) \}$:
\begin{align}
   \label{eq:ncvmp1}
   \widetilde{L}(\vmu_i,\vmu_{/i}) &=   \sum_{a\in\mathbb{N}_i} \myexpect_{q} [ \log \tp_{nc}^{a, i} ] + \ang{\vmu_i,  \sum_{a\in\mathbb{N}_i} \widetilde{\veta}_{ai} - \vlambda_i} + A_i(\vlambda_i)  + \textrm{constant} \\
                                   &=   \sum_{a\in\mathbb{N}_i} \myexpect_{q} [ \log \tp_{nc}^{a, i} ] + \myexpect_{q_i} [\log (\tp_c^{i}/ q_i)]  + \textrm{constant}
\end{align}
where $\tp_c^{i}$ is a conjugate factor whose natural parameter is equal to $\sum_{a\in\mathbb{N}_i} \widetilde{\veta}_{ai}$.
Therefore, we can simply use Lemma 1 to 3 to simplify. 

Using the results of Lemma 3, we get the following expression:
   \begin{align}
   q_{i,t+1} &\propto \sqr{ \exp\crl{\left\langle{\boldsymbol{\phi}_i(\mathbf{z}_i), \sum_{a\in\mathbb{N}_i} \widehat{\nabla}_{\mu_i} \myexpect_q[ \log \tp_{nc}^{a,i}] \rvert_{\mu=\mu_t} } \right\rangle } \tp_c^i }^{\beta_t} (q_{i,t})^{1-\beta_t} \\
      &=\sqr{ \exp\crl{\left\langle{\boldsymbol{\phi}_i(\mathbf{z}_i), \sum_{a\in\mathbb{N}_i} \sqr{ \widehat{\nabla}_{\mu_i} \myexpect_q[ \log \tp_{nc}^{a,i}] \rvert_{\mu=\mu_t} + \widetilde{\veta}_{ai} }  } \right\rangle} }^{\beta_t} (q_{i,t})^{1-\beta_t}
   \end{align}
   We define the natural parameter of the approximation term in the exponential:
   \begin{align}
      \tvlambda_{i,t} = \sum_{a\in \mathbb{N}_i} \sqr{ \tilde{\veta}_{ai} + \widehat{\nabla}_{\mu_i} \myexpect_{q_t} [\log \tp_{nc}^{a,i}] \rvert_{\mu=\mu_t} }
   \end{align}
   The natural parameter of $q_{t+1}$ is obtained by taking a convex combination of $\tvlambda_{i,t}$ and the natural parameter of $q_t$, i.e., $\vlambda_{i,t}$:
   \begin{align}
      \vlambda_{i,t+1} = \beta_t \tvlambda_{i,t} + (1-\beta_t) \vlambda_{i,t}
   \end{align}

   \subsection{Equivalence to NC-VMP}
   \label{app:ncvmp}
 We can show that NC-VMP is equivalent to our method under these conditions: the gradients w.r.t. the mean are exact and the step-size is set to 1, i.e., $\beta_t = 1$.
 We now present a formal proof.

We rewrite the lower bound w.r.t. $\vmu_i$ shown in \eqref{eq:ncvmp1}:
\begin{align}
   \widetilde{L}(\vmu_i,\vmu_{/i}) &=   \sum_{a\in\mathbb{N}_i} \myexpect_{q} [ \log \tp_{nc}^{a, i} ] + \left\langle{\vmu_i,  \sum_{a\in\mathbb{N}_i} \widetilde{\veta}_{ai} - \vlambda_i} \right\rangle + A_i(\vlambda_i)  + \textrm{constant} 
\end{align}
By taking the derivative w.r.t. $\vmu_i$ using \eqref{eq:grad_mu_1}, we get the first line below. 
\begin{align}
   \nabla_{\mu_i} \widetilde{L}(\vmu_i,\vmu_{/i}) &=   \sum_{a\in\mathbb{N}_i} \nabla_{\mu_i} \myexpect_{q} [ \log \tp_{nc}^{a, i} ] + \sum_{a\in\mathbb{N}_i} \widetilde{\veta}_{ai} -\vlambda_i 
\end{align}
We define the conjugate factor with natural parameter $\widetilde{\veta}_{ai}$ by $\tp_c^{ai}$. We use the property that the gradient of a conjugate-exponential term, such as $\myexpect_q[\log \tp_c^{ai}]$ w.r.t. $\vmu_i$ is equal to the term itself. We derived this while proving Lemma \ref{lemma:grad_conj} in Appendix \ref{app:proof_thm1} (although it is easy to prove by simply substituting the definition of $\tp_c^{ai}$). Therefore in the second term, we can simply
substitute the gradient of $\myexpect_q[\log \tp_c^{ai}]$ to get the following:
\begin{align}
    \nabla_{\mu_i} \widetilde{L}(\vmu_i,\vmu_{/i}) &=   \sum_{a\in\mathbb{N}_i} \nabla_{\mu_i} \myexpect_{q} [ \log \tp_{nc}^{a, i} ] + \sum_{a\in\mathbb{N}_i} \nabla_{\mu_i} \myexpect_q [\log \tp_c^{a,i} ]  -\vlambda_i \\
                                                  &=   \sum_{a\in\mathbb{N}_i} \nabla_{\mu_i} \myexpect_{q}[ \log p(\vx_a|\vx_{\pa_a}) ]   -\vlambda_i 
\end{align}
where the last line is obtain by using Assumption 4.

We also note that the derivative of the Bregman divergence term $\mathbb{B}_{A^*_i}(\vmu_i\|\vmu_{i,t})$ is equal to $\vlambda_i -\vlambda_{i,t}$.
\begin{align}
   \nabla_{\mu_i} \mathbb{B}_{A_i^*}(\vmu_i\|\vmu_{i,t}) &= \nabla_{\mu_i} \sqr{ A_i^*(\vmu_i) - A_i^*(\vmu_{i,t}) - \ang{\vmu_i-\vmu_{i,t}, \nabla A_i^*(\vmu_{i,t})}} \\
                                                       &= \nabla_{\mu_i} A_i^*(\vmu_i) - \nabla_{\mu_i} A_i^*(\vmu_{i,t}) \\
                                                       &= \vlambda_i -\vlambda_{i,t}
\end{align}
When we use $\beta_t$ = 1, mirror descent reduces to the following:
\begin{align}
   \max_{\mu_i} { \left\langle \vmu_i, \widehat{\nabla}_{\mu_i} \widetilde{\elbofinal} (\vmu_t) \right\rangle 
   -  \mathbb{B}_{A^*}(\vmu_i \| \vmu_{i,t}) }
\end{align}
Taking the derivative w.r.t. $\vmu_i$ and setting it to zero, we get:
\begin{align}
   \vlambda_{i,t+1} &= \sum_{a\in\mathbb{N}_i} \nabla_{\mu_i} \myexpect_{q}[ \log p(\vx_a|\vx_{\pa_a}) ] \rvert_{\mu=\mu_t} = \sum_{a\in\mathbb{N}_i} \vC_{i,t}^{-1} \nabla_{\lambda_i} \myexpect_{q}[ \log p(\vx_a|\vx_{\pa_a}) ] \rvert_{\lambda=\lambda_t}
   \end{align}
   where $\vC_{i,t}$ is the Fisher information matrix of $q_{i,t}$. This is exactly the message used in NC-VMP.

\section{Dataset Details}
\label{app:dataset}
Datasets for Bayesian logistic regression is available at {\footnotesize \url{https://www.csie.ntu.edu.tw/~cjlin/libsvmtools/datasets/binary.html}}, for gamma factor model can be found at {\footnotesize \url{https://github.com/davidaknowles/gamma_sgvb}}, and for Gaussian-process classification can be obtained from {\footnotesize \url{https://github.com/emtiyaz/prox-grad-svi}}. 

For all experiments, we first use grid search to tune model hyper-parameters and then fix them during our experiments. The statistics of the datasets and the model hyper-parameters used are given in Table \ref{data_stat}.

\begin{table}[h]
\center
\caption{A list of models and datasets. $N_{\text{Train}}$ is the number of training data. $K$ is the number of factors. The last column shows the values of hyperparameters. The details of the hyperparameters can be found in Appendix \ref{app:ex_cvi}. For GP classification $\sigma_f$ and $l$ are hyperparameters of the squared-exponential kernel.
}
\begin{tabular}{|l|l||l|l|l|l|c|}
\hline
Model & Dataset & $N$ & $D$ & $N_{\text{Train}}$   & Hyperparameters \\
\hline
\multirow{3}{*}{Bayesian Logistic Regression} & a1a & 32,561 & 123 & 1,605  & $\delta=2.8072$    \\
& a7a & 32,561 & 123 & 16,100 & $\delta=5.0$    \\
& Colon-cancer & 62 & 2000 & 31 &   $\delta=596.3623$  \\
& Australian-scale &  690  &  14 &  345 &   $\delta=10^{-5}$  \\
& Breast-cancer-scale & 683  &  10 & 341  &   $\delta=1.0$  \\
& Covtype-binary-scale & 581,012  &  54 & 290,506  &   $\delta=0.002$  \\
\hline
\hline
\multirow{1}{*}{Gamma Factor Model} & Cytof & 522,656 & 40 & 300,000   &  $\sigma^2=0.1$, $K=40$, $a=b=1.0$  \\
\hline
\hline
\multirow{1}{*}{Gamma Matrix Factorization} & MNIST & 70,000 & 784 & 60,000   &  $a^{(z0)}=b^{(z0)}=a^{(w0)}=0.1$ \\
\multirow{1}{*}{} & & & &  & $b^{(z0)}=0.3$, $K=100$ \\ 
\hline
\hline
\multirow{1}{*}{Gaussian Process Classification} & USPS3vs5 & 1,781 & 256 & 884 &  $\log(\sigma_f)= 5.0$, $\log(l)=2.5$\\
\hline
\end{tabular}
\label{data_stat}
\end{table}

\section{Algorithmic Details and Additional Results}
\label{app:algo_details}
In this section, we include 3 additional methods in our comparisons. We compare to a method called PG-SVI which is similar to the PG-exact method but uses stochastic gradients are used. Similarly, we also compare to a method called CVI-exact which is similar to CVI but uses exact gradients.
For GP classification, we compare to expectation propagation (EP).

Table \ref{alg_stat} gives the details of algorithmic parameters used in our experiments. 

\begin{table}[h]
\center
\caption{Algorithmic Parameters and Model Parameters } 
\begin{tabular}{|l|l|l|l|}
\hline
Model & Datasets   & step size & MC samples \\
\hline
\multicolumn{4}{|c|}{CVI-exact, PG-exact, CVI, S\&K Alg2, S\&K FG ($\beta=\frac{w}{1+w}$)}\\
\hline
 & Colon-cancer  &   $w=0.3$ & 10 \\
Bayesian Logistic Regression & Australian-scale &   $w=0.4$   & 10 \\
  & a1a    & $w=0.4$  & 10 \\
    & a7a    & $w=0.4$  & 10 \\
 & Breast-cancer-scale   & $w=0.3$  & 10\\
  & Covtype-scale    & $w=0.3$  & 10\\
\hline
\hline
\multicolumn{4}{|c|}{Knowles, CVI, where $w_0$ denotes the initial step size in Knowles (Ada-delta)}\\
\hline
  &  &    $w_0=10.0$ (Knowles)  &  \\
Gamma Factor Model & Cytof    &  $\beta=5 \times 10^{-5}$ (CVI)  & 50 \\
\hline
\hline
\multicolumn{4}{|c|}{ADAM, CVI, where $w_0$ denotes the initial step size in ADAM}\\
\hline
  &  &    $w_0=0.5$ (ADAM)  &  \\
Gamma Matrix Factorization & MNIST   &  $\beta=0.02$ (CVI)  & 10 \\
\hline
\hline
\multicolumn{4}{|c|}{CVI-exact, PG-exact, CVI, PG-SVI  ($\beta=\frac{w}{1+w}$)}\\
\hline
   &     & $w=1.0$ (CVI-exact, PG-exact) & \\
Gaussian Process Classification &  USPS3vs5 &   $w=0.3$ (CVI,PC-SVI ) & 100\\
\hline
\end{tabular}
\label{alg_stat}
\end{table}

\subsection{Additional Results}
\label{app:add}
We compare Bayesian logistic regression on seven real datasets. The results are summarized in Table \ref{fig:BLR2}. 
All methods reach the same performance. Chol is the slowest method. When $D>N$ S\&K-FG is supposed to perform better than S\&K-Alg2, but the situation is reversed when $N>D$. PG-Exact and CVI-exact are expected to have the same performance. CVI is expected to be a faster than them because stochastic
gradients might be cheaper to compute. It is also expected to perform well for both $N>D$
regime and $D>N$ regime.

Additional results for the gamma factor model and gamma matrix factorization model are in Table \ref{fig:GFM2} and \ref{fig:GFM3} respectively.
 
For GP Classification, we present results below where we compare our method (CVI) to the following methods:
expectation propagation (EP), explicit optimization with LBFGS using Cholskey factorization (Chol), Proximal gradient methods (PG-SVI).
For PG-SVI and CVI, we use MC approximation to compute gradient while for CVI-exact, we use exact gradient. Figure \ref{fig:GP} shows the result of Gaussian Process Classification. 
\begin{figure*}[h]
\center
\includegraphics[width = 3in]{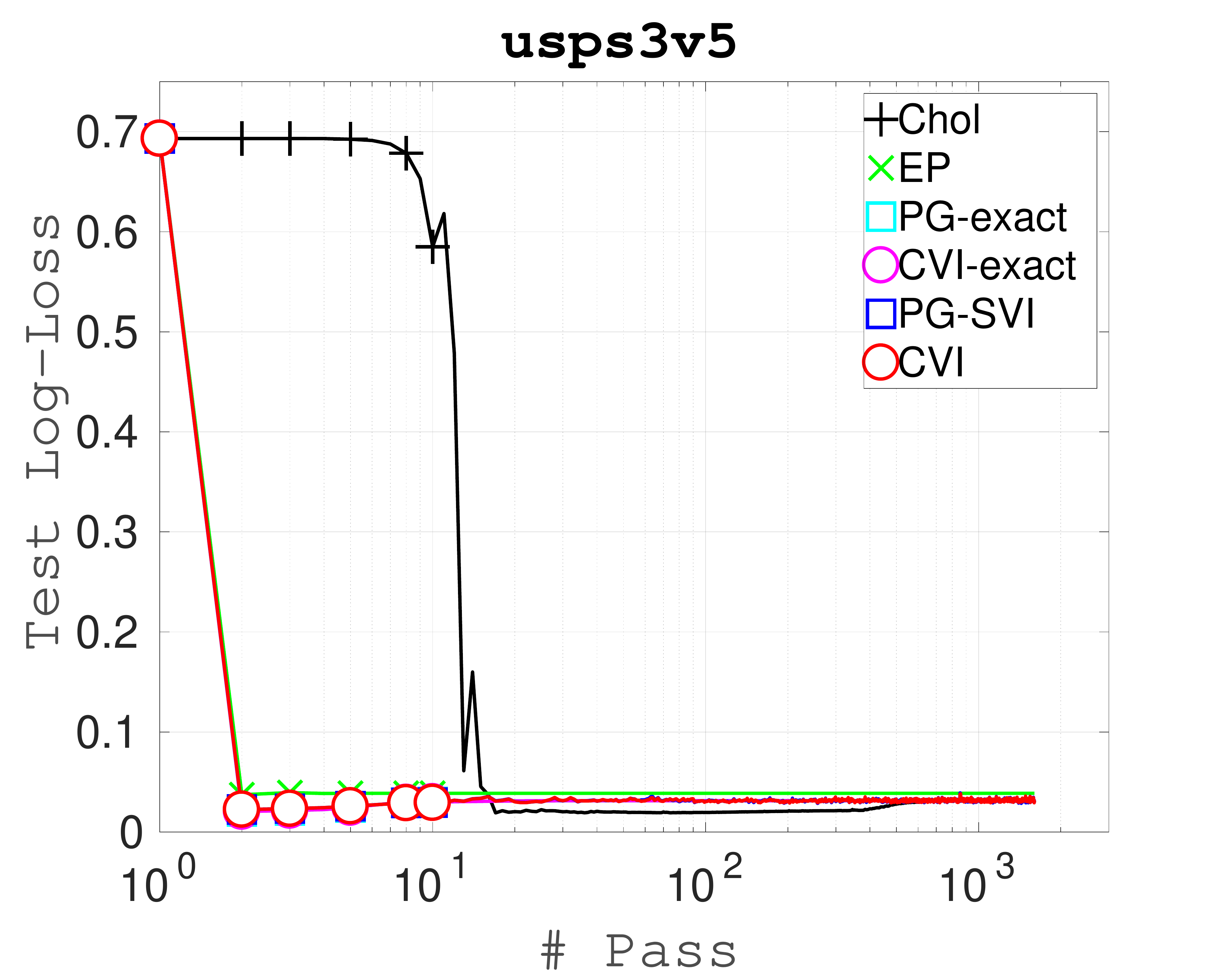}
\caption{Comparison on Gaussian Process Classification.}
\label{fig:GP}
\end{figure*}

\begin{table}[h]
\center
\caption{A summary of the results obtained on Bayesian logistic regression. In all columns, a lower value implies better performance. We report total time of convergence (note that for stochastic methods this is difficult and due to this the reported time might be longer than expected).}
\label{fig:BLR2}
\begin{tabular}{|l||l|l|l|l|}
\hline
Dataset & Methods & Neg-Log-Lik& Log Loss& Time \\
\hline
\multirow{4}{*}{a1a $(N>D)$} 
& Chol & 591.4  & 0.49 &  0.82s \\
& S\&K Alg2 & 590.5 & 0.49 & 0.07s \\
& S\&K FG & 590.5 & 0.49 & 0.09s \\
& PG-exact &  591.6 & 0.49 & 0.15s \\
& CVI-exact   & 590.5 & 0.49 & 0.10s \\
& CVI  & 590.4  & 0.49 & 0.10s \\
\hline
\hline
\multirow{4}{*}{a7a $(N>D)$} 
& Chol &     5,418.1 & 0.47 & 17.79s    \\
& S\&K Alg2 &   5,416.4 &  0.47 &   0.74s \\
& S\&K FG &   5,416.3 &  0.47 & 1.19s   \\
& PG-exact &   5,418.0 &   0.47 &  1.35s \\
& CVI-exact  &   5,416.3 &   0.47 &  1.17s \\
& CVI &  5,416.3  &   0.47 &  0.95s  \\
\hline
\hline
\multirow{4}{*}{Colon-cancer $(D>N)$} 
& Chol &    18.26   & 0.694 &  93.229s  \\
& S\&K Alg2 &   18.26  &  0.693 & 6.142s   \\
& S\&K FG & 18.26   & 0.693  &  0.026s   \\
& PG-exact &   18.25  &    0.696 &  0.052s  \\
& CVI-exact  &   18.26  &    0.698 &  0.012s \\
& CVI  &  18.26   &    0.698 & 0.021s   \\
\hline
\hline
\multirow{4}{*}{Australian-scale $(N>D)$} 
& Chol &  191.62  & 0.473 &   0.193s  \\
& S\&K Alg2 & 190.99    &  0.480 & 0.013s   \\
& S\&K FG & 190.95    & 0.479  &   0.034s  \\
& PG-exact &   191.57 &    0.479 &  0.056s   \\
& CVI-exact   & 191.14   & 0.480    & 0.020s  \\
& CVI &  191.30   &  0.478  & 0.011s   \\
\hline
\hline
\multirow{4}{*}{Breast-cancer-scale $(N>D)$} 
& Chol &  34.21   &  0.139 & 0.110s   \\
& S\&K Alg2 & 34.20     & 0.139   & 0.014s   \\
& S\&K FG & 34.15    &  0.137 &  0.036s   \\
& PG-exact &  34.18   &    0.138  & 0.063s   \\
& CVI-exact   &  34.24    & 0.138     & 0.032s  \\
& CVI  &   34.15  &   0.140 &  0.021s  \\
\hline
\hline
\multirow{4}{*}{Covtype-scale $(N>D)$ but $N$ is large} 
& Chol &  149,641   &  0.7404 & 198.1932s   \\
& S\&K Alg2 & 149,623    & 0.7403  & 56.7972s   \\
& S\&K FG & 149,612  &  0.7403 &  20.309s   \\
& PG-exact &  149,615  &    0.7403 & 42.6777s   \\
& CVI-exact  &  149,615     & 0.7403     & 39.5720s  \\
& CVI  &   149,616  &   0.7403 &  14.3319s  \\
\hline
 \hline
\end{tabular}
\end{table}


\begin{table}[t]
\center
\caption{Results obtained on Gamma factor model, a lower value implies better performance. CVI is much faster than Knowles method.}
\label{fig:GFM2}
\begin{tabular}{|l||l|l|l|l|}
\hline
Dataset & Methods &   Log Loss& Time \\
\hline
\multirow{2}{1pt}{Cytof} & Knowles    & 52.25   &  210.03s   \\
& CVI   &   52.52  & 50.91s    \\
 \hline
\end{tabular}
\end{table}
 
\begin{table}[t]
\center
\caption{Results obtained on Gamma Matrix Factorization, a lower value implies better performance. CVI outperforms ADAM.}
\label{fig:GFM3}
\begin{tabular}{|l||l|l|l|l|}
\hline
Dataset & Methods &  Test Loss & Time \\
\hline
\multirow{2}{1pt}{MNIST} & ADAM    & 0.000125   &  1776.83s   \\
& CVI   &   0.000119  & 1692.64s    \\
 \hline
\end{tabular}
\end{table}

\section{Details of the Gamma Factor Model}
\label{sec:gamma2}
We consider the model discussed by \cite{knowles2015stochastic}. In this model, observations $\vy_{i} \in \mathbb{R}^D$, $i=1 \dots N$ are modeled as
\begin{align}
   p(\vY,\vZ|\sigma^2,a,b) = p(\vY|\vZ) p(\vZ) = {\sqr{ \prod_{i=1}^{N} p(\vy_{i}|\vZ,\sigma^2)}} { \sqr{ \prod_{j=1}^{D}  \prod_{k=1}^{K} p(Z_{jk}|a,b)}}
\end{align}
where each column of $\vY$ follows $p(\vy_{i}|\vZ,\sigma^2)=\gauss(\vy_{i}|0, \vZ \vZ^T+\sigma^2 I)$ and each element of $\vZ$ follows $p(Z_{jk}|a,b)=\Ga(Z_{jk}|a,b)$ with the following parameterization $\Ga(x|\alpha,\beta) \propto x^{\alpha-1} e^{-x\beta}$.

This is a non-conjugate model since the data term $p(\vy|\vZ)$ is not conjugate to the prior $p(\vZ)$.
We choose the following mean-field approximation: 
$$q(\vZ)= \prod_{i=1}^{N} \prod_{j=1}^{D} q(Z_{j,k}).$$
where each factor is a Gamma distribution.


\section{Details of the Gamma Matrix Factorization}
\label{app:def}
Given the data matrix $\vX$ of size $V\times N$, the Gamma matrix-factorization assumes the following joint-distribution:
\begin{equation}
\begin{split}
p(\vX,\vW,\vZ) = \prod_{i=1}^V &\sqr{\prod_{j=1}^N p(X_{i,j}|\vw_i^T\vz_j) } \\
&\times \sqr{ \prod_{i=1}^V \prod_{k=1}^K  \Ga(w_{k,i} | a^{(w0)}_{k,i},b^{(w0)}_{k,i}) }
\sqr{ \prod_{j=1}^N \prod_{k=1}^K \Ga(z_{k,j}|a^{(z0)}_{k,j},b^{(z0)}_{k,j}) }
\end{split}
\end{equation}
where $\vw_i,\vz_j$ are $K$ dimensional latent vectors, $\vW$ and $\vZ$ are $K\times V$ and $K\times N$ matrices respectively.
The likelihood term $p(X_{i,j}|\vw_i^T\vz_j)$ is a Poisson distribution. We use the following gamma posterior:
\begin{align}
q(\vW,\vZ) = \sqr{ \prod_{i=1}^V \prod_{k=1}^K   \Ga(w_{k.i} |a^{(w)}_{k,i},b^{(w)}_{k,i}) } \sqr{ \prod_{j=1}^N  \prod_{k=1}^K  \Ga(z_{k,j}|a^{(z)}_{k,j},b^{(z)}_{k,j}) }
\end{align}
%
%
%

\end{appendix}
\end{document}